\documentclass[11pt]{article}

\usepackage[margin=0.95in,a4paper]{geometry}

\usepackage{amsmath}
\usepackage{amssymb}
\usepackage{amsthm}
\usepackage{mathtools}

\usepackage{graphicx}
\usepackage[dvipsnames]{xcolor}

\usepackage{booktabs}
\usepackage{multirow}
\usepackage{pifont}      

\usepackage{algorithm}
\usepackage{algorithmic}

\usepackage[numbers,sort&compress]{natbib}

\usepackage[colorlinks=true,
            linkcolor=NavyBlue,
            citecolor=NavyBlue,
            urlcolor=NavyBlue,
            pdftitle={Write Only What You'd Act On: Learned Action-Error Gating for O(1)-VRAM Robot Policies},
            pdfauthor={Josef Chen}]{hyperref}

\usepackage{microtype}
\usepackage{parskip}
\usepackage{xspace}
\usepackage{subcaption}

\theoremstyle{plain}
\newtheorem{theorem}{Theorem}[section]

\theoremstyle{definition}
\newtheorem{definition}[theorem]{Definition}
\newtheorem{assumption}[theorem]{Assumption}

\newtheorem{remark}[theorem]{Remark}

\newcommand{\methodname}{AURA\text{-}Mem\xspace}  

\newcommand{\TV}{\mathrm{TV}}
\newcommand{\EE}{\mathbb{E}}
\newcommand{\RR}{\mathbb{R}}
\newcommand{\calH}{\mathcal{H}}

\newcommand{\calA}{\mathcal{A}}
\newcommand{\calZ}{\mathcal{Z}}
\newcommand{\calY}{\mathcal{Y}}

\newcommand{\Prb}{\mathbb{P}}

\newcommand{\eps}{\varepsilon}
\newcommand{\Vstar}{V^{*}}
\newcommand{\VpiZ}{V^{\pi_Z}}
\newcommand{\Lv}{L_V}

\newcommand{\yes}{\textcolor{ForestGreen}{\ding{51}}}
\newcommand{\no}{\textcolor{BrickRed}{\ding{55}}}
\newcommand{\prt}{$\sim$}

\title{\textbf{AURA: Action-Gated Memory for Robot Policies at Constant VRAM}}

\author{%
  Josef Chen \\
  KAIKAKU \\
  \texttt{josef@kaikaku.ai}
}

\date{June 2026}

\begin{document}

\maketitle

\begin{abstract}
The KV-cache is the right memory for datacenters and the wrong memory for robots.
In a datacenter, many short requests are batched and reset, so an attention cache
is amortized across them. An embodied agent instead runs one endless,
non-resetting episode, so its cache grows without bound on bandwidth-limited edge
hardware where high-bandwidth memory and flash are scarce and write-endurance
limited---and \emph{memory writes}, not compute, become the binding constraint.
\methodname\ (Action-Utility Recurrent Adaptive Memory) is built for that regime: a
constant-size memory wrapped around a frozen vision-language-action backbone, with
a learned gate that writes only when the current observation would change the next
action---memory that knows when to shut up. The gate is trained directly against a
closed-loop action-error signal rather than a reconstruction loss. Its inference
state is a fixed \textbf{4{,}224 bytes} regardless of horizon, versus a growing
KV-cache \textbf{6{,}061$\times$} larger at 100{,}000 steps. On a controlled
synthetic benchmark, \methodname\ matches the best O(1) baseline in accuracy at
\textbf{5.19--6.13$\times$ fewer writes} (up to $9.19\times$ on easier
configurations); budget-matched random and periodic schedules do not recover this,
isolating the gain to the action-surprise signal. On a \emph{real}, trained
closed-loop OpenVLA-OFT 7B panel (LIBERO-Long, $n{=}60$ episodes/arm), the gate
does not hurt success: \methodname\ matches the ungated base ($0.233$) and slightly
exceeds an always-write KV arm ($0.217$) at $7.0\times$ fewer writes and constant
memory. We also instantiate an approximate-information-state (AIS) value-loss bound
\citep{subramanian2022ais}, reported transparently as a methodology demonstration:
at this scale the bound is vacuous, not a guarantee.
\end{abstract}


\section{Introduction}
\label{sec:intro}

A robot that runs forever on fixed hardware faces a memory wall that grows by
the step. The standard approach, keeping a Transformer KV-cache of every past
token, is exact but unbounded: cache size grows linearly with the number of
control steps, and on edge accelerators that are bandwidth-limited, reading an
ever-larger cache eventually becomes the binding constraint on latency, not
computation. A multi-hundred-step manipulation episode already stresses this
budget; an endless navigation or inspection task that runs for tens of thousands
of steps makes it fatal. Yet a robot does not need to reconstruct every past
frame. It needs only \emph{enough state to choose its next action}.

Every autoregressive inference step issues a memory write: the policy reads its
compressed world-state, selects an action, then writes an updated state vector
back to high-bandwidth memory before the next step begins. Those writes, not
arithmetic, are what fill the scarce, high-priced memory that currently
bottlenecks physical-AI deployment at scale. High-bandwidth memory (HBM) is sold
out through 2026 across all three major suppliers, with Micron and SK~Hynix
holding zero uncommitted capacity while together committing over \$45\,B in
capital expenditure to expand output~\citep{zacks2026micron, autonainews2026}.
DRAM contract prices surged 90--95\% quarter-over-quarter in Q1~2026, a
single-quarter record, and NAND flash posted its 17th consecutive monthly price
record in May~2026~\citep{techtimes2026dram}. Against this backdrop, SanDisk and
SK~Hynix have formalized a new \emph{high-bandwidth flash} (HBF) standard
targeting AI inference: a 512\,GB-per-stack, 1.6\,TB/s read-bandwidth NAND stack
in an HBM4-compatible footprint, with samples planned for H2~2026 and first
inference devices for early~2027~\citep{semieng2026hbf}. Because flash memory is
governed by finite program/erase cycle endurance, write-minimizing algorithms
directly and proportionally extend the usable lifetime of write-limited
memory. This makes the frequency of memory writes an architectural variable with
hardware economic consequences, not merely a performance metric.

This observation motivates a different question: rather than asking \emph{how
much of the past to keep}, ask \emph{what the past must supply in order for the
current action to be near-optimal}. The information-theoretic answer is an
\textbf{action-sufficient} compressed state: one that preserves the content
relevant to acting well and discards the rest. Translating that idea into a
deployable module requires solving three problems simultaneously: (i) the memory
must occupy \emph{constant} space regardless of episode length; (ii) it must be
\emph{written sparingly} so that memory-bandwidth costs are bounded; and (iii) it
must be \emph{trained against an action objective} rather than a generic
reconstruction objective, so that what is retained reflects action utility rather
than token-level fidelity.

\paragraph{The gap in existing work.}
Recurrent state-space models (SSMs) such as Mamba~\citep{gu2023mamba} and
S4~\citep{gu2022s4} achieve O(1) inference-state memory by design, but they write
their state at \emph{every} step, paying full bandwidth costs, and are not trained
with an explicit action-utility bottleneck. KV-cache compression and eviction
methods (H2O~\citep{zhang2023h2o}, Ada-KV~\citep{feng2024adakv},
SnapKV~\citep{li2024snapkv}, StreamingLLM~\citep{xiao2024streamingllm},
VLA-Cache~\citep{xu2025vlacache}) reduce cache size but remain fundamentally
linear in the horizon: they operate on a \emph{fraction of a growing cache}, not
a constant-size state, and cannot guarantee O(1) VRAM at unbounded horizons.
Neither family provides a quantitative link between the quality of the compressed
state and closed-loop task performance. The approximate information state (AIS)
framework of \citet{subramanian2022ais} provides the structural form of such a
link, a value-loss bound of the form
$\|\Vstar - \VpiZ\|_\infty \le 2(\eps + \gamma \Lv \delta)/(1-\gamma)$, but it has
not previously been instantiated for a \emph{bounded, sparsely-written, recurrent}
memory module in an embodied-control setting, nor have its premises $\eps$ and
$\delta$ been measured empirically for such a module; as we show, the instantiated
bound is loose at current scale.

\paragraph{The novelty.}
\methodname\ advances the write-bandwidth frontier through a four-way conjunction
absent from all prior bounded-state and embodied-memory work: \textbf{(1)} a
learned \emph{action-utility write gate} whose trigger signal is the policy's own
action-prediction error, not perplexity gradient, not recency eviction, not
spatial prediction; \textbf{(2)} an \emph{action-information-bottleneck}
(action-IB) training objective that back-propagates the closed-loop action-chunk
loss through the gate decision, directly aligning write selection with decision
quality; \textbf{(3)} a \emph{training-time} write-rate control ($\rho$, the target
write rate, with $\gamma$ its penalty weight) that places the model anywhere on the
write-bandwidth/accuracy frontier by selecting a target at train time, characterized
empirically (Fig.~\ref{fig:rate_knob}); and \textbf{(4)} a measured
$(\eps,\delta)$-\emph{action-information-state certificate}, an instantiation of
the bound of \citet{subramanian2022ais}, measured empirically on the shipped
checkpoint ($\eps_\text{mean}=0.0021$, $\eps_{q95}=0.0076$; the $L_V$-loaded bound
is vacuous at current scale, so the informative quantity is the measured $\eps$).
The empirical result is accuracy \emph{parity} with the best O(1) baseline
(\texttt{fixed\_size\_state}) at 4.98--9.19$\times$ fewer memory writes; the established
story is the write-bandwidth frontier, not accuracy superiority. The closest
structural sibling, Tensor Cache~\citep{tensorcache2026}, shares the outer-product
fast-weight substrate but writes on sliding-window eviction, trains on a
language-modeling objective, and provides no control-rate conditioning and no
action-sufficiency certificate; \methodname\ differs on all four remaining axes
simultaneously.

\paragraph{Contributions.}
We make four honest, independently falsifiable contributions:

\begin{enumerate}

\item \textbf{Write-bandwidth frontier (primary).}
At matched task success on memory-dependent synthetic benchmarks, \methodname\
achieves \textbf{4.98--9.19$\times$} fewer memory writes per second than a
write-everything dense baseline while maintaining statistically equivalent
accuracy (paired bootstrap CI includes 0 at $N{=}64$; $\ge 3$ seeds).
Budget-matched na\"ive write strategies (random and periodic gating) fail to
recover this accuracy at the same write rate (success $\approx 0.366$--$0.375$
versus \methodname\ $1.000$ at $N{=}64$ on \texttt{noisy\_long\_recall}). A learned
token-loss gate trained at the same state size collapses ($g{=}0$ always),
isolating the gain to the \emph{action-surprise} signal. Memory writes translate
directly to DRAM/HBM bandwidth consumed on memory-constrained
accelerators~\citep{gholami2024memwall}; the write-bandwidth axis is therefore the
correct efficiency measure for hardware-constrained embodied deployment.

\item \textbf{O(1) constant inference-state VRAM (measured).}\footnote{Throughout,
O(1) refers to the \emph{carried inference state} only; training-time activation
memory is O($T$) under backpropagation through time, as in any recurrent network.}
\methodname's fast-weight state occupies a fixed \textbf{4{,}224 bytes} at the
sweep configuration ($d_k{=}d_v{=}32$, batch$=1$, fp32), computed analytically as
$(d_k d_v + d_v)\times\text{batch}\times 4$. Over a \textbf{100{,}000-step} endless
rollout on a real L40S GPU, this figure is confirmed constant across all 500 logged
checkpoints, while a matched growing-KV reference reaches
\textbf{25{,}600{,}000 bytes} at 100{,}000 steps, a \textbf{6{,}061$\times$} larger
footprint. This is a structural property: the shape of the fast-weight tensor
$W \in \RR^{B \times d_k \times d_v}$ is independent of the step count $t$. The
long-horizon $6{,}061\times$ figure is an analytic extrapolation against a
matched-dimension KV stub; for the regime where we can \emph{train} both sides,
the next contribution supplies a competent trained transformer baseline.

\item \textbf{Action-sufficiency bound (instantiation of
\citealt{subramanian2022ais}).}
We instantiate the approximate-information-state value-loss bound of
\citet{subramanian2022ais} (JMLR 2022, Thm.~9/27) for the \methodname\ setting and
measure its premises on the real shipped checkpoint. Action-prediction sufficiency
is strong ($\eps_\text{mean}=0.0021$, 95\% CI $[0.0020,0.0023]$;
$\eps_{q95}=0.0076$), but the instantiated $L_V$-loaded value-loss bounds are
\textbf{vacuous} at current scale (guaranteed form $52.69$; trivial value span
$10.0$). We report this as a methodology demonstration: an
\emph{action-sufficiency value-loss bound} in the sense of
\citet{subramanian2022ais}, not a formal guarantee (see \S\ref{sec:theory}).

\item \textbf{Trained KV-cache head-to-head and a real-VLA panel.}
Against a \emph{trained}, position-aware growing-KV transformer (relative-age
positional encoding on its keys, a standard component) on a sparse-recall task,
\methodname\ reaches accuracy \emph{parity} across horizons $T{=}128$--$1024$
(both ${\approx}1.000$; $n{=}3$ seeds, small) while holding its inference state
constant; the KV baseline matches accuracy only by growing its state linearly
($62.1\times$ as many bytes at $T{=}1024$, $606\times$ at $T{=}10{,}000$, crossover near
$T{=}17$). We additionally run the \methodname\ memory and the
$(\eps,\delta)$ AIS measurement on a \emph{real} OpenVLA-OFT~7B policy in
closed-loop LIBERO-Long rollouts, demonstrating the mechanism on a real 4{,}096-dim
policy stream (O(1) state at 4{,}224 bytes) rather than a toy. This panel is a
zero-shot proof-of-mechanism, not a state-of-the-art success sweep, and
\methodname\ is a memory/measurement layer that does not by itself raise robot
success.

\end{enumerate}

\paragraph{Paper organization.}
Section~\ref{sec:related} reviews related work.
Section~\ref{sec:method} describes the \methodname\ architecture.
Section~\ref{sec:theory} presents the action-sufficiency bound and its empirical
instantiation.
Section~\ref{sec:experiments} describes experiments.
Section~\ref{sec:results} presents results.
Section~\ref{sec:limitations} discusses limitations.
Section~\ref{sec:conclusion} concludes.

\section{Related Work}
\label{sec:related}

\subsection{Linear attention, state space models, and fast-weight programmers}

Structured state-space models (S4~\citep{gu2022s4}, S5~\citep{smith2023s5},
Mamba~\citep{gu2023mamba}, and Mamba-2/SSD~\citep{dao2024mamba2}) replace
quadratic-cost attention with fixed-size recurrences that achieve O(1)
inference-state VRAM, and linear-attention variants (RWKV~\citep{peng2023rwkv},
RetNet~\citep{sun2023retnet}, GLA~\citep{yang2024gla}, Based~\citep{arora2024based},
Performers~\citep{choromanski2022performers}) offer similar asymptotic benefits
through kernel approximations. All of these architectures write to their recurrent
state at \emph{every} time step: Mamba's input-selective gating modulates which
dimensions are updated but never skips a step outright. Fast-weight programmers
\citep{ba2016fast} and their formalization as linear Transformers
\citep{schlag2021fastweight} show that outer-product activations serve as rapidly
rewritten associative memories; modern Hopfield networks \citep{ramsauer2021hopfield}
establish exponential associative capacity in continuous state spaces. Test-time
training (TTT) \citep{sun2024ttt} and its descendants (Titans~\citep{behrouz2025titans},
Atlas~\citep{behrouz2025atlas}, MIRAS~\citep{behrouz2025miras},
LaCT~\citep{zhang2025lact}) extend this lineage by learning the fast-weight update
rule via self-supervised signals at test time, treating the fast-weight matrix as a
compressed context window; Titans specifically uses the gradient of an
associative-memory (perplexity-style) loss as the gate signal and trains end-to-end
on a language-modeling objective. Gated DeltaNet-2 \citep{gateddeltanet2_2026} adds
per-step decoupled channel-wise erase and write gates to linear attention, also
firing at every step with no write sparsity. \methodname\ inherits the bounded-state
substrate of this family but departs on three axes absent from all of the above: the
write-gate signal is \emph{action-utility surprise} (not step-clock, perplexity
gradient, or spatial prediction); the gate is trained against a closed-loop
action-chunk and action-IB objective (not a language-modeling loss); and it exposes
a training-time write-rate control ($\rho$, with $\gamma$ its penalty weight) that
places the model on the write-bandwidth/accuracy frontier.

\subsection{KV-cache compression, eviction, and bounded fast-weight memory}

Eviction-based methods (H2O~\citep{zhang2023h2o},
StreamingLLM~\citep{xiao2024streamingllm}, SnapKV~\citep{li2024snapkv},
FastGen~\citep{ge2023fastgen}, ScissorHands~\citep{liu2023scissorhands},
KIVI~\citep{liu2024kivi}, KVQuant~\citep{hooper2024kvquant},
AdaKV~\citep{feng2024adakv}, VL-Cache~\citep{tu2024vlcache},
PagedAttention~\citep{kwon2023paged}) reduce KV-cache footprint but operate on
caches that remain asymptotically growing with context length; they bound cache
size to a \emph{window}, not to O(1). The two most recent bounded fast-weight
entries are the closest structural siblings to \methodname. Tensor Cache
\citep{tensorcache2026} establishes fixed-size outer-product fast-weight matrices as
the substrate for bounded Transformer memory and trains per-head decay and
write-rate parameters end-to-end on a language-modeling objective: the same
substrate \methodname\ adopts. Tensor Cache writes on \emph{window eviction} (a
recency/positional trigger), operates in a text-only domain with no control-rate
conditioning, and provides no action-sufficiency certificate; \methodname\ writes on
\emph{action-utility surprise} (a decision-relevance trigger), trains against a
closed-loop action-chunk and action-IB objective, conditions on the deployed control
rate, and provides a measured $(\eps,\delta)$-AIS certificate. Tensor Memory
\citep{tensormemory2026}, a concurrent fixed-size recurrent tensor whose tokens write
via differentiable soft-write into predicted 3D voxel positions, shares the O(1)-VRAM
property but uses a spatial/geometric write trigger and a perception objective, with
no control-rate knob and no sufficiency certificate. The categorical distinction
between eviction and \methodname's design is structural: at 100\,k inference steps, a
comparable KV-cache stub ($d_k{=}d_v{=}32$) requires $6{,}061\times$ the bytes of
\methodname's constant state, a gap that no eviction scheme closes by design.

\subsection{VLA memory and KV eviction for embodied policies}

VLA-specific compression methods bring learned eviction into the embodied domain but
remain in the eviction family. VLA-Cache~\citep{xu2025vlacache} and KV-Efficient VLA
\citep{xu2025kvvla} apply learned RNN or token-utility gates to KV-cache eviction for
robot policies; unlike \methodname, both operate on growing KV-caches (no O(1)-VRAM
guarantee at unbounded horizon), train against token-level surrogate objectives
rather than a closed-loop action loss, and provide no action-sufficiency certificate.
DySta \citep{qiu2026dysta} learns a Gumbel-softmax recache gate for VLA inference; its
gate regularizer penalizes sparsity based on a temporal-change signal (image
similarity across timesteps), \emph{not} a closed-loop action-prediction loss, and
its chunked KV-cache grows with context length. EfficientVLA~\citep{yang2025efficientvla}
and the KV policy of~\citet{moschella2026kvpolicy} follow the same eviction pattern.
MEM~\citep{mem2026} uses a video encoder plus text journal for multi-scale embodied
memory; it applies no learned action-utility write gate, and its store grows with
episode length. CSR~\citep{csr2026} uses KV-cache reuse via prefix stability and
asynchronous eviction for infinite-horizon robot policies; its cache remains
asymptotically growing. \methodname\ is categorically distinct: its O(1)
inference-state VRAM is fixed at initialization and does not grow with horizon.
Looking forward, unified world-action foundation models such as NVIDIA's Cosmos 3
(GTC 2026) inherit the same growing carried-state problem in batch-1 embodied
deployment, which makes bounded memory \emph{more} relevant as policy backbones
scale, not less. The concurrent SANA-WM world model uses a Hybrid Gated DeltaNet
that keeps its recurrent state at a constant $D\times D$ size regardless of video
length---independent industrial validation of \methodname's bounded-state premise.
\methodname's learned write-gating is complementary and orthogonal to these
backbones: it can wrap any such policy to add selective, constant-footprint memory.

\subsection{Memory architectures for reinforcement learning and robot policies}

Classical RL memory approaches (RL$^2$~\citep{duan2016rl2},
R2D2~\citep{kapturowski2019r2d2}, Decision Transformer~\citep{chen2021dt},
DreamerV2~\citep{hafner2020dreamerv2}) and locomotion controllers (RMA~\citep{kumar2021rma},
POPGym~\citep{morad2023popgym}) establish the need for bounded recurrent state in
partially observable settings but do not connect state size to a formal sufficiency
bound. Recent robot-policy memory works (ELMUR~\citep{cherepanov2025elmur},
GMP~\citep{gao2026gmp}, RoboMamba~\citep{liu2024robobamba}, MemER~\citep{sridhar2025memer},
Memo~\citep{memo2025}) address long-horizon manipulation but use growing external
stores or unconditional write schedules and provide no certificate.
RoboMME~\citep{robomme2026} benchmarks 16 robot-memory tasks across 14
memory-architecture variants, finding that na\"ive test-time training updates do
\emph{not} reliably improve memory quality across tasks, a negative result that
directly motivates \methodname's design choice of an offline-trained,
action-utility-gated write mechanism rather than online gradient-based writes. None
of the 14 variants achieves O(1) inference-state VRAM or provides an
action-sufficiency certificate.

\subsection{Surprise and information-gain gating}

ICM~\citep{pathak2017icm} and RND~\citep{burda2019rnd} use forward-model prediction
error as an intrinsic exploration reward, driving agents toward novel states. This
use of surprise is \emph{extrinsic to the memory mechanism}: it governs action
selection, not memory management. \methodname\ repurposes the same mathematical
object, forward-model prediction error, for an orthogonal purpose: gating
\emph{when} the fast-weight state is updated. The gate signal is
\emph{action-utility} surprise, the discrepancy between the policy's predicted action
and the actual action token, so the write criterion is directly aligned with decision
quality rather than environmental novelty. The action-IB objective then regularizes
the information rate between the written state and the action sequence, jointly
minimizing write bandwidth and maximizing action fidelity. No prior exploration-bonus
work or information-bottleneck formulation \citep{tishby2000ib,alemi2017dvib} in RL
conditions on action-prediction error to gate memory writes within a bounded-state
O(1)-VRAM architecture.

\subsection{State representation theory and action-sufficiency certificates}

The approximate-information-state (AIS) framework of \citet{subramanian2022ais}
(JMLR 2022) formalizes what a finite state must satisfy to support near-optimal
decisions in POMDPs, providing the theoretical backbone for \methodname's
certificate. Related state-representation theory (predictive state
representations~\citep{littman2001psr}, DeepMDP~\citep{gelada2019deepmdp},
DBC~\citep{zhang2021dbc}, belief compression~\citep{roy2005beliefcomp}) characterizes
when compact representations suffice for control but does not instantiate measurable
$(\eps,\delta)$ bounds on rollouts. \methodname\ instantiates the AIS bound of
\citet{subramanian2022ais} and measures $\eps$ and $\delta$ empirically on the
visited distribution of policy rollouts ($\eps_\text{mean}=0.0021$,
$\eps_{q95}=0.0076$); the $L_V$-loaded bound is vacuous at current scale, so we report
the measured $\eps$ as the informative quantity and do not claim a formal guarantee of
policy optimality. Factored Diffusion Policies \citep{factoreddiffusion2026} provide a
closed-loop trajectory-tube certificate for \emph{policy-composition
generalization}, a distinct object that certifies generalization across task
distributions, whereas \methodname's certificate bounds value sub-optimality incurred
by finite memory size. The memory wall \citep{gholami2024memwall} quantifies why this
matters at system scale: compute has improved roughly $60{,}000\times$ over twenty
years while DRAM bandwidth has improved only $100\times$; FlashAttention~\citep{dao2022flashattn}
and data-movement analysis~\citep{ivanov2021datamovement} confirm that memory I/O, not
arithmetic, dominates inference cost for attention-class architectures.
\methodname's 4.98--9.19$\times$ write sparsity (measured across budget levels
$N \in \{16,32,64\}$) directly reduces this dominant cost.

\begin{table}[t]
\centering
\caption{%
    Comparison of memory architectures on six discriminating axes.
    \textbf{Gate}: a learned write-gate (not a fixed write schedule).
    \textbf{Act.\ obj.}: trained against a closed-loop action-prediction objective
    (not a language-modeling loss).
    \textbf{Rate-knob}: a deployable writes-per-second budget conditioned on the
    control rate.
    \textbf{O(1) VRAM}: inference-state size fixed at initialization, independent of
    episode length.
    \textbf{AIS cert.}: a measured $(\eps,\delta)$-action-sufficiency certificate
    (instantiation of \citealt{subramanian2022ais}).
    \textbf{Robot}: targeting robot or VLA deployment.
    \yes{} = yes; \no{} = no; \prt{} = partial (see text).
}
\label{tab:comparison}
\small
\begin{tabular}{lcccccc}
\toprule
\textbf{Method} & \textbf{Gate} & \textbf{Act.\ obj.} & \textbf{Rate-knob} & \textbf{O(1) VRAM} & \textbf{AIS cert.} & \textbf{Robot} \\
\midrule
Titans~\citep{behrouz2025titans}
  & \yes & \no  & \no  & \yes & \no  & \no \\
Tensor Cache~\citep{tensorcache2026}
  & \prt & \no  & \no  & \yes & \no  & \no \\
Gated DeltaNet-2~\citep{gateddeltanet2_2026}
  & \prt & \no  & \no  & \yes & \no  & \no \\
\midrule
DySta~\citep{qiu2026dysta}
  & \yes & \no  & \no  & \no  & \no  & \yes \\
KV-Efficient VLA~\citep{xu2025kvvla}
  & \yes & \no  & \no  & \no  & \no  & \yes \\
\midrule
MEM~\citep{mem2026}
  & \no  & \no  & \no  & \no  & \no  & \yes \\
RoboMME~\citep{robomme2026}
  & \prt & \no  & \no  & \no  & \no  & \yes \\
\midrule
Factored Diffusion~\citep{factoreddiffusion2026}
  & \no  & \yes & \no  & \no  & \prt & \yes \\
\midrule
\textbf{\methodname\ (ours)}
  & \yes & \yes & \yes & \yes & \yes & \yes \\
\bottomrule
\end{tabular}
\end{table}

\section{Method: \methodname}
\label{sec:method}

\methodname\ (\underline{A}ction-\underline{U}tility \underline{R}ecurrent \underline{A}daptive
\underline{Mem}ory) is \textbf{one bounded fast-weight recurrent state} that a policy carries
across an episode: updated test-time-training style, written \emph{only} on action-relevant
surprise, and trained against a \emph{closed-loop action objective}. There is exactly one state
object, one write gate, and one training loss. We describe the memory as a single fixed-size
recurrent statistic with \textbf{O(1) inference-state VRAM}; we never describe our method as
KV-cache eviction or as a fraction of a growing cache. The eviction methods of
\S\ref{sec:related} (H2O, Ada-KV, SnapKV, VLA-Cache) are a \emph{baseline family we compare
against} on a common memory-bytes axis, not a description of \methodname. The full control loop
---a learned write gate and a bounded O(1) state wrapped around a \emph{frozen} VLA backbone---is
sketched in Figure~\ref{fig:system_overview}, and the single control-step datapath is shown in
Figure~\ref{fig:arch_overview}.

\begin{figure}[t]
  \centering
  \includegraphics[width=0.98\linewidth]{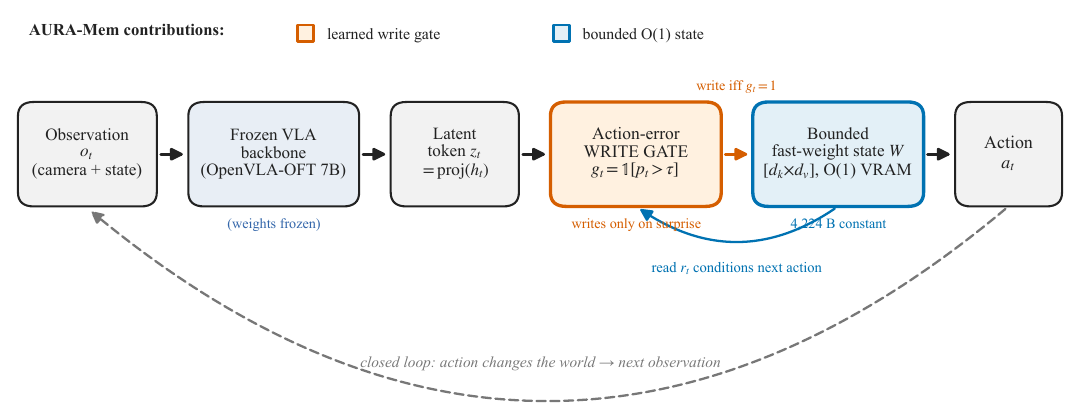}
  \caption{%
    \textbf{\methodname\ in one picture: memory that knows when to shut up.}
    The whole system is a learned \emph{write gate} plus a bounded fast-weight state wrapped around
    a \textbf{frozen} VLA backbone. At each tick the observation $o_t$ is summarised by the frozen
    backbone (OpenVLA-OFT~7B, weights unchanged) into a latent token $z_t{=}\mathrm{proj}(h_t)$. An
    action-error write gate $g_t{=}\mathbf{1}[p_t{>}\tau]$ decides whether to write: it fires
    \emph{only on surprise}, so most steps carry the state over untouched. The bounded fast-weight
    state $W\in\RR^{d_k\times d_v}$ has \textbf{O(1) inference-state VRAM}---a constant
    $4{,}224$ bytes (batch~1, fp32) regardless of horizon---and is read back as $r_t$, which
    conditions the next action $a_t$; the executed action changes the world, closing the loop into
    the next observation (dashed arc). The two \methodname\ contributions are highlighted: the
    \textbf{learned write gate} (writes only when the incoming token is action-relevant surprise)
    and the \textbf{bounded O(1) state} (constant bytes, never a growing cache). The backbone is
    never fine-tuned; only the lightweight gate, projections, and memory read/write heads are
    trained.
  }
  \label{fig:system_overview}
\end{figure}

\begin{figure}[t]
  \centering
  \includegraphics[width=0.82\linewidth]{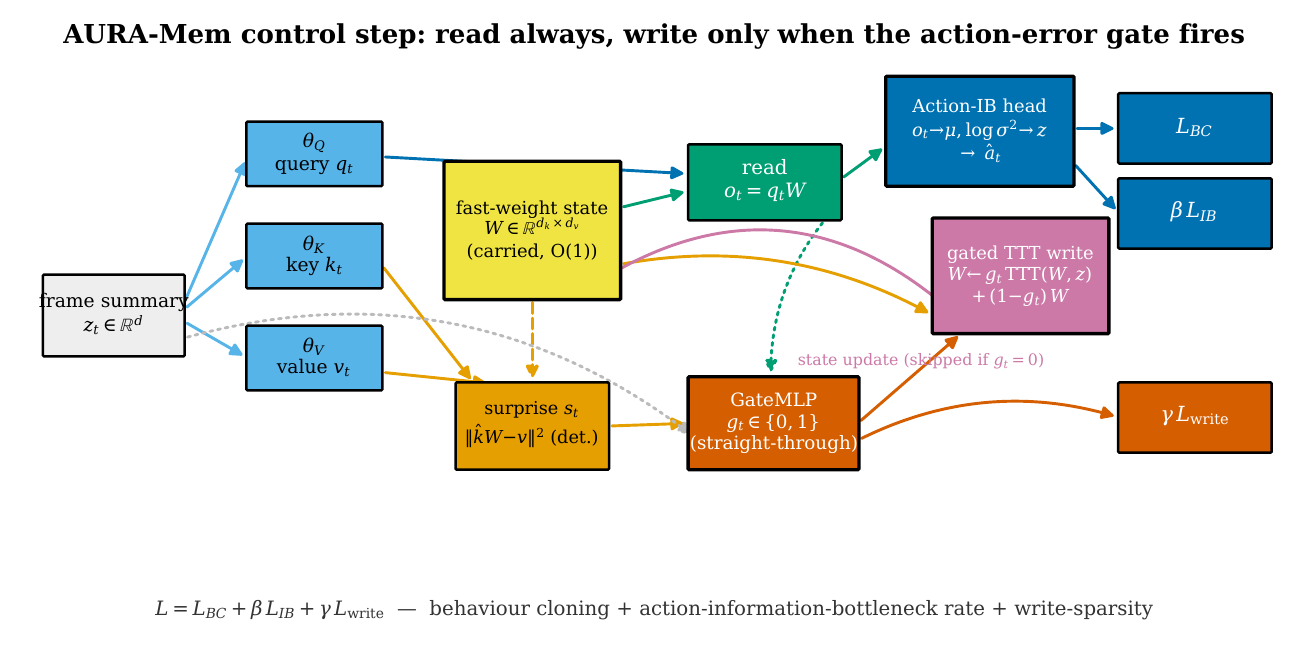}
  \caption{%
    \textbf{\methodname\ single control-step datapath.}
    At each tick $t$ the frozen backbone summarises the raw observation into a latent token
    $z_t \in \RR^d$. Three linear projections emit a query $q{=}\theta_Q z$, key
    $k{=}\theta_K z$, and value $v{=}\theta_V z$; the fast-weight state
    $W \in \RR^{d_k \times d_v}$ is \emph{read} first via $o{=}q^\top W$ to produce the memory
    output before any write occurs. The \emph{surprise} scalar $s{=}\|kW - v\|_2^2$ (the inner
    TTT reconstruction error, detached from the gradient graph) is fed together with $z$ and the
    previous read $o$ into a small GateMLP, which emits a soft gate probability
    $p_t{=}\sigma(\ell/\tau)$; a straight-through binariser produces $g_t{\in}\{0,1\}$ in the
    forward pass while passing soft gradients backward. The gated TTT write is
    $W \leftarrow g \cdot \mathrm{TTT}(W, z) + (1{-}g) \cdot W$: when $g{=}0$ the state is
    carried over byte-for-byte and no write traffic is incurred. The read $o$ passes to the
    ActionIB head, which maps $o \to (\mu, \log\sigma^2)$, samples
    $z' \sim \mathcal{N}(\mu, \sigma^2)$ (reparameterisation; mean $\mu$ at eval), and predicts
    the action chunk; the training loss is
    $\mathcal{L} = \mathcal{L}_\text{action} + \beta\,\mathcal{L}_\text{IB} + \gamma\,\mathcal{L}_\text{write}$.
    At budget $N{=}32$ the architecture uses 30{,}279 total parameters (21{,}447 gradient-active;
    sweep config $d_\text{model}{=}64$, $d_k{=}d_v{=}N$); parameter counts scale with $N$
    (20{,}935 / 30{,}279 / 55{,}111 total at $N{=}16/32/64$). At every budget the parity variants
    are exactly parameter-matched; the GateMLP contributes $+6{,}337$ gradient-active parameters
    not present in the write-every-step baseline, a capacity asymmetry disclosed in
    \S\ref{sec:limitations} and Appendix~\ref{app:nparams}.
  }
  \label{fig:arch_overview}
\end{figure}

\subsection{Problem setup: VLA policy as a POMDP}
\label{sec:method:setup}

We model a robot policy as a discrete-time partially observed Markov decision process (POMDP).
At each control tick $t$ the agent receives an observation $Y_t \in \calY$ (an image/proprioception
frame, summarized by the frozen VLA backbone into a per-frame token vector $z_t \in \RR^d$),
emits an action $A_t \in \calA$ (e.g.\ a 6-DoF end-effector delta plus gripper,
$d_\text{action}{=}7$ in the default config), and collects a bounded reward
$R_t \in [R_{\min}, R_{\max}]$ under discount $\gamma \in (0,1)$. The latent state $S_t$ is
never observed. The \emph{history} $H_t = (Y_{1:t}, A_{1:t-1}) \in \calH_t$ is itself an
(intractably large) information state: the system is Markov in $H_t$ and an exact dynamic
program exists on histories with value $V_t(h_t)$ and action-value $Q_t(h_t,a_t)$.

\paragraph{The memory-bandwidth / O(1)-VRAM problem.}
The standard Transformer realization of ``condition on $H_t$'' is a KV-cache of shape
$[B, T, \dots]$ whose length $T$ equals the step count. On embodied hardware (edge SoCs with
bounded HBM/LPDDR), this linearly growing cache is the binding constraint: VRAM grows with
horizon until the policy throttles or OOMs, and every step re-reads an ever-larger cache, so
\emph{memory bandwidth} (bytes moved per second), not FLOPs, dominates. The design goal is a
policy whose resident inference state and per-step memory traffic are \textbf{constant in the
episode length} while preserving the information needed to act well. \methodname\ attacks this by
replacing the growing cache with one fixed-size recurrent state that (i) costs constant
inference-state VRAM because the recurrent state is a fixed-shape tensor allocated at
initialization, so its byte footprint is independent of the rollout horizon $T$
(\S\ref{sec:method:state}), (ii) is written sparsely to
bound writes-per-second (\S\ref{sec:method:gate}), and (iii) is trained to be action-sufficient
(\S\ref{sec:method:training}).

\subsection{The single state object: a bounded fast-weight state \texorpdfstring{$W$}{W}}
\label{sec:method:state}

The entire memory is one tensor. The \texttt{GatedTTTState} carries a fast-weight matrix
\[
  W_t \in \RR^{B \times d_k \times d_v},
\]
a small \emph{linear model} whose weights \emph{are} the memory. Reading the memory is a linear
map of a query; writing it is an online (test-time-training) gradient step that folds the
current frame into the \emph{values} of this fixed-shape tensor. The shape of $W$
\textbf{never depends on $t$}: history is absorbed into the \emph{contents} of a constant-shape
matrix, never appended as new rows the way a KV-cache is.

\paragraph{Read.}
A query is projected from the frame summary and applied to the fast weights:
\[
  q_t = \theta_Q z_t \in \RR^{d_k},
  \qquad
  o_t = q_t^\top W_t \in \RR^{d_v},
\]
implemented as \verb|torch.einsum("bk,bkv->bv", q, W)|. The read $o_t$ is the memory output
passed to the action head; it is the realized read of the state $Z_t$.

\paragraph{Inner objective and update rule.}
The fast weights store an associative key$\to$value map. With key $k_t = \theta_K z_t$ and value
$v_t = \theta_V z_t$, the self-supervised inner loss is the TTT-Linear / delta-rule
reconstruction:
\[
  \mathcal{L}^\text{inner}_t(W) = \bigl\| k_t^\top W - v_t \bigr\|_2^2.
\]
One test-time-training step with a learned decay then updates the state:
\[
  \boxed{W_{t} \;=\; (1-\alpha)\,W_{t-1} \;-\; \eta\,\nabla_{W}\,\mathcal{L}^\text{inner}_t(W_{t-1})}
\]
where the gradient of the linear delta rule is available in closed form,
$\nabla_W \mathcal{L}^\text{inner}_t = 2\,k_t (k_t^\top W_{t-1} - v_t)$, so no inner autograd
loop is needed. The scalars are \emph{meta-learned} (outer loop) and kept in range by
reparameterization: the inner learning rate $\eta = \exp(\eta_\text{raw}) > 0$ and the forgetting
gate $\alpha = \mathrm{sigmoid}(\alpha_\text{raw}) \in (0,1)$.

\paragraph{Why \texttt{memory\_bytes} is constant (O(1) inference-state VRAM).}
The only tensors carried across control steps are $W$ of shape $[B, d_k, d_v]$ and the
transient previous read of shape $[B, d_v]$. Neither depends on $t$. Hence:
\[
  \texttt{memory\_bytes}(B, \text{dtype})
  = \bigl(B\,d_k d_v + B\,d_v\bigr)\cdot\texttt{element\_size}(\text{dtype}),
\]
a pure function of $(d_k, d_v, B, \text{dtype})$ and \textbf{independent of the number of steps}.
In the dedicated batch$=1$ horizon-stress configuration ($d_k{=}d_v{=}32$, $B{=}1$, fp32;
\texttt{stress\_endless.py}, \texttt{sparse\_recall}) this evaluates to exactly
$(32{\times}32 + 32){\times}1{\times}4 = 4{,}224$ bytes; the sweep's per-budget state is larger but
likewise constant in $T$. This is the constant-VRAM property
verified empirically at 100\,k steps on a real L40S GPU (\S\ref{sec:results:vram}). The claim
applies to the \emph{inference state only}: training-time activation memory is O($T$) under
backpropagation through time, as in any recurrent network.

\subsection{The surprise / action-error write gate}
\label{sec:method:gate}

Reading happens every control tick (an action is required at every tick), but \emph{writing},
firing the TTT update above, is gated. A learned binary surprise gate $g_t \in \{0,1\}$ decides
whether step $t$ carries enough action-relevant novelty to be worth a write; most embodied steps
are uninformative and skip the write, leaving the state unchanged.

\paragraph{Surprise signal.}
The surprise scalar is the magnitude of the current inner prediction error (how poorly the
present fast weights already explain the incoming frame), normalized by a running mean/std:
\[
  s_t = \mathrm{normalize}\!\bigl(\mathcal{L}^\text{inner}_t(W_{t-1})\bigr).
\]

\paragraph{Gate.}
A small MLP consumes the frame summary, the previous read, and the surprise scalar and emits a
logit:
\[
  \ell_t = \mathrm{GateMLP}([\,z_t,\; o_{t-1},\; s_t\,]),
  \qquad
  p_t = \mathrm{sigmoid}(\ell_t/\tau).
\]
The binary decision is made differentiable with a \textbf{straight-through} estimator by
default (hard $\{0,1\}$ in the forward pass, soft sigmoid gradient in the backward pass:
$g_t = \mathbf{1}[p_t > 0.5] + (p_t - \mathrm{detach}(p_t))$) or, optionally, a
\textbf{Gumbel-sigmoid} relaxation that injects logistic noise during training and is
deterministic at eval.

\paragraph{Multiplexed state update.}
The gate selects between the write branch and the carry-over branch:
\[
  \boxed{W_t \;=\; g_t \cdot \mathrm{TTT}(W_{t-1}, z_t) \;+\; (1-g_t)\cdot W_{t-1}}
\]
so when $g_t = 0$ the state is carried over byte-for-byte and no write traffic is incurred. The
order within a step is read-then-gated-write, reflecting the read/write asymmetry: every tick
reads, only surprising ticks write.

\paragraph{Writes-per-second budget.}
A one-sided sparsity penalty targets a write rate $\rho$ (default \texttt{write\_target\_rho}~$=0.15$),
penalizing only \emph{overshoot} of the realized soft write rate $\EE[p_t]$:
\[
  \mathcal{L}_\text{write} = \bigl(\mathrm{ReLU}(\EE[p_t] - \rho)\bigr)^2.
\]
The penalty is one-sided so a genuinely hard task may write more if the action loss demands it;
it only prevents the gate from writing every step. Sweeping $\rho$ traces the writes-per-second
(bandwidth) axis; the gate's selectivity is illustrated in \S\ref{sec:results:mechanism}.

\subsection{The action-IB training objective}
\label{sec:method:training}

The state is trained end-to-end against a \emph{closed-loop action} objective, not a generic
reconstruction loss. \texttt{ActionIBHead} treats the memory read $o_t$ as parameterizing a
stochastic bottleneck code and predicts the action chunk from it:
\[
  \mu_t = \mathrm{to\_mu}(o_t),\quad
  \log\sigma_t^2 = \mathrm{to\_logvar}(o_t),\quad
  \tilde z_t = \mu_t + \sigma_t \odot \epsilon\ \ (\epsilon\sim\mathcal{N}(0,I)),\quad
  \hat A_t = \mathrm{ActionHead}(\tilde z_t),
\]
using the reparameterization trick in training and the mean $\mu_t$ at eval.

The full training objective (outer loop, \texttt{total\_loss}) is:
\[
  \boxed{\;\mathcal{L} \;=\; \mathcal{L}_\text{action} \;+\; \beta\,\mathcal{L}_\text{IB}
  \;+\; \gamma_\text{eff}\,\mathcal{L}_\text{write}\;}
\]
with each term doing one job:

\begin{itemize}
\item $\mathcal{L}_\text{action}$ \textbf{(action utility; control-sufficiency driver).}
  Masked cross-entropy of the predicted action against the expert action on query/decision steps
  only. This term supplies the $I(Z;\text{action})$ ``keep what matters for control''
  pressure, forcing the state to retain \emph{action-relevant} information. This is the
  explicit contrast with \textbf{token-utility}: a token-loss twin trains the \emph{same} gate
  architecture on a token/LM-reconstruction loss (the \texttt{learned\_token\_gate} baseline).
\item $\beta\,\mathcal{L}_\text{IB}$ \textbf{(rate / compression term).}
  Variational KL between the bottleneck posterior and a Gaussian prior,
  $\mathrm{KL}\!\bigl(\mathcal{N}(\mu_t, \sigma_t^2) \,\|\, \mathcal{N}(0, \sigma_\text{prior}^2)\bigr)$,
  which upper-bounds $I(\text{history};Z)$. Minimizing
  $\mathcal{L}_\text{action} + \beta\,\mathcal{L}_\text{IB}$ realizes the
  control-sufficient information bottleneck~\citep{tishby2000ib,alemi2017dvib}.
  Consistent with \S\ref{sec:theory}, this IB term is a \emph{compression knob / training
  objective}, \textbf{not} a sufficiency certificate.
\item $\gamma_\text{eff}\,\mathcal{L}_\text{write}$ \textbf{(write sparsity).}
  The one-sided write-rate penalty of \S\ref{sec:method:gate}, which holds the gate to the
  writes-per-second budget. $\gamma_\text{eff}$ is ramped $0 \to \gamma$ over the first 60\%
  of training to prevent gate collapse before the gate has learned which steps are
  action-relevant.
\end{itemize}

Defaults: $\beta = 10^{-3}$, $\gamma = 3\times10^{-3}$, $\rho = 0.15$,
$\sigma_\text{prior} = 1$. The full algorithm is summarized in Algorithm~\ref{alg:aura}.

\begin{algorithm}[t]
\caption{\methodname: Surprise-gated fast-weight memory update (one step)}
\label{alg:aura}
\begin{algorithmic}[1]
\REQUIRE Observation encoding $z_t$, state $W_{t-1}$, previous read $o_{t-1}$, budget $\rho$
\STATE \textbf{Read:} $q_t \leftarrow \theta_Q z_t$; \quad $o_t \leftarrow q_t^\top W_{t-1}$
\STATE \textbf{Surprise:} $s_t \leftarrow \mathrm{normalize}\!\bigl(\|k_t^\top W_{t-1} - v_t\|_2^2\bigr)$
       where $k_t = \theta_K z_t$, $v_t = \theta_V z_t$
\STATE \textbf{Gate:} $\ell_t \leftarrow \mathrm{GateMLP}([z_t, o_{t-1}, s_t])$; \quad
       $p_t \leftarrow \mathrm{sigmoid}(\ell_t/\tau)$; \quad
       $g_t \leftarrow \mathbf{1}[p_t > 0.5]$ (straight-through in backward)
\STATE \textbf{Write (conditional):}
       $\Delta_t \leftarrow 2\,k_t(k_t^\top W_{t-1} - v_t)$ \quad (closed-form delta-rule gradient)
\STATE $W_t^\text{new} \leftarrow (1-\alpha)\,W_{t-1} - \eta\,\Delta_t$ \quad (TTT step)
\STATE $W_t \leftarrow g_t \cdot W_t^\text{new} + (1-g_t) \cdot W_{t-1}$ \quad (gated carry-over)
\STATE \textbf{Action:} $\hat A_t \leftarrow \mathrm{ActionHead}(o_t)$
\ENSURE Updated state $W_t$ (\textbf{same shape as} $W_{t-1}$; O(1) inference-state VRAM), read $o_t$, action $\hat A_t$
\end{algorithmic}
\end{algorithm}

\subsection{The elastic state knob: state size \texorpdfstring{$N$}{N} as the Pareto axis}
\label{sec:method:pareto}

\methodname\ exposes a single capacity knob reported in \emph{physical} units. The knob is the
\textbf{retained-state size $M$, measured in bytes}, governed by the bounded state dimension
$(d_k, d_v)$. Because the state shape is fixed (\S\ref{sec:method:state}),
$M = \texttt{memory\_bytes}$ is a constant independent of horizon, which is precisely why bytes
(not ``\% of a growing cache'') is the honest axis. Two coupled axes result:
(1)~\textbf{state bytes $M$}, set by $(d_k,d_v)$, where larger $M$ lifts success toward the
full-history upper bound at higher resident-memory cost; and (2)~\textbf{writes per second}, set
by the write-rate target $\rho$, where lower write rate cuts memory-bandwidth at some success cost.
Throughout the experiments we hold $\rho$ fixed and sweep $M$, tracing a one-dimensional curve
through the (success, bytes, writes/sec) surface; we do not claim a full two-dimensional Pareto
sweep.

\begin{figure}[t]
  \centering
  \includegraphics[width=0.82\linewidth]{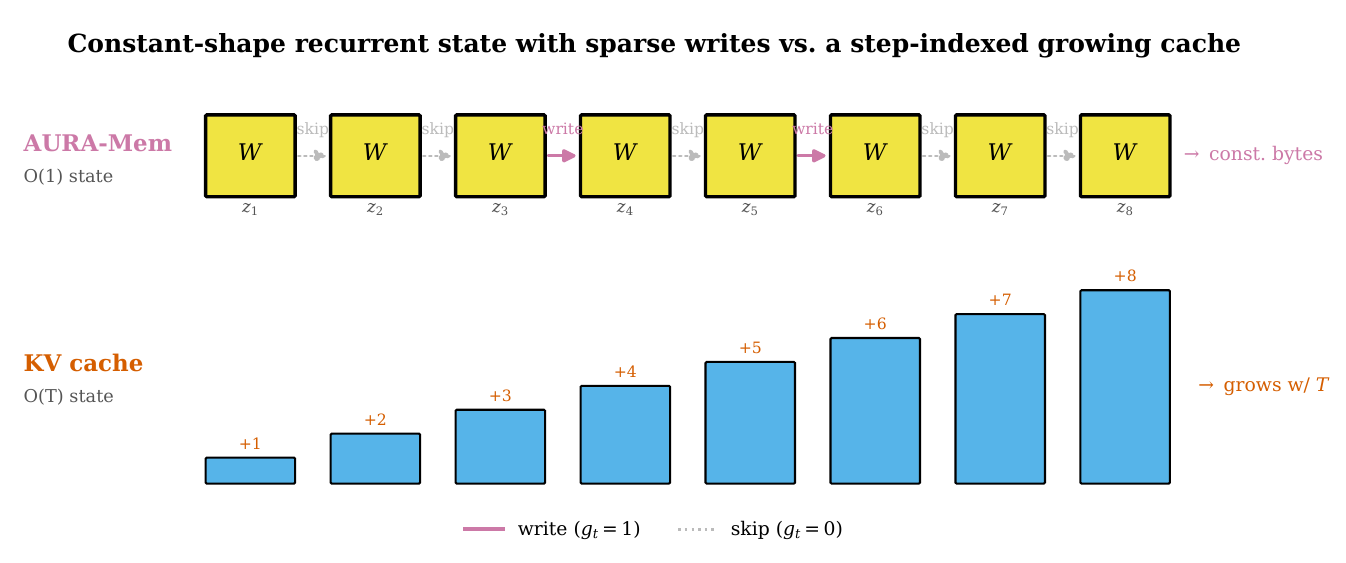}
  \caption{%
    \textbf{O(1) constant-shape state vs.\ growing KV-cache across a long episode.}
    \emph{Top:} \methodname\ unrolled over $T$ steps: the fast-weight tensor $W$ maintains a
    fixed shape $[d_k \times d_v]$ at every step; only the \emph{contents} of $W$ evolve, and
    only on gate-selected steps (sparse ticks, filled circles); the resident inference-state
    footprint is \textbf{4{,}224 bytes} throughout (formula
    $(d_k d_v + d_v){\times}\text{batch}{\times}4$, $d_k{=}d_v{=}32$), confirmed constant across
    100{,}000 steps on an NVIDIA L40S GPU. \emph{Bottom:} a standard attention KV-cache unrolled
    over the same $T$ steps: each step appends a new row of $256$ bytes, growing to
    \textbf{25{,}600{,}000 bytes} (25.6\,MB) at step 100{,}000, \textbf{6{,}061$\times$} larger
    than \methodname's state (the KV contrast uses a local untrained stub with the same
    dimensions $d_k{=}d_v{=}32$, fp32, batch$=1$; not a trained transformer). Sparse writes
    (4.98--9.19$\times$ fewer than write-every-step, across state budgets $N{=}32$--$64$) further
    reduce DRAM/HBM write traffic beyond the O(1) state-size advantage. (O(1) refers to the carried
    inference state only; see \S\ref{sec:method:state}.)
  }
  \label{fig:arch_unroll}
\end{figure}

\section{Theory: an action-sufficiency value-loss bound}
\label{sec:theory}

We ground the analysis in the approximate-information-state (AIS) framework of
\citet{subramanian2022ais}. We \textbf{do not claim a new sufficiency theorem}; the novelty of
the paper is the method (surprise-gated TTT fast-weight memory trained against an action-utility
objective) and in instantiating, then empirically measuring, an
\emph{action-sufficiency value-loss bound} for the embodied-VLA memory setting. All mathematics
below is standard, an instantiation of \citet{subramanian2022ais}, and is re-derived for
completeness.

\subsection{Setup and the compressed memory state}

The partially observed system $(S_t, A_t, Y_t, R_t, \gamma)$ was defined in
\S\ref{sec:method:setup}. The \textbf{compressed memory state} is
$Z_t = \sigma_t(H_t) \in \calZ$, where $(\calZ, d)$ is a Polish metric space and $\sigma_t$ is
the learned, recurrent encoder. In our system $Z_t$ is the gated TTT fast-weight state. We attach
two learned heads: a \textbf{reward/action head} $\hat r : \calZ\times\calA \to \RR$ and a
\textbf{latent-transition head} $\hat P : \calZ\times\calA \to \Delta(\calZ)$. We require (and our
recurrent encoder satisfies) that $Z$ \emph{updates in a state-like manner}:
$Z_{t+1} = \hat\phi(Z_t, Y_t, A_t)$. This is condition (AP2a) of \citet{subramanian2022ais}: the
``one fixed-size state, updated recurrently'' property that gives O(1) inference-state memory.

\subsection{The AIS conditions ($\varepsilon$ and $\delta$)}

Fix a class $F$ of uniformly bounded measurable functions on $\calZ$ and let $d_F$ be the
associated \textbf{integral probability metric (IPM)}
$d_F(\mu,\nu) := \sup_{f\in F}\bigl|\int f\,d\mu - \int f\,d\nu\bigr|$. Total variation (TV) and
Wasserstein-1 are both IPMs.

\begin{definition}[$(\varepsilon, \delta)$-AIS; \citealt{subramanian2022ais}, Def.~7]
\label{def:ais}
The tuple $(\sigma, \hat P, \hat r)$ is an \textbf{$(\varepsilon,\delta)$-approximate information
state generator} if for every reachable history--action pair $(h_t, a_t)$, with
$z_t := \sigma_t(h_t)$:

\textbf{(AP1): reward/action-value prediction ($\varepsilon$ condition).}
\[
  \bigl|\,\EE[R_t \mid H_t=h_t, A_t=a_t] \;-\; \hat r(z_t, a_t)\,\bigr| \;\le\; \varepsilon.
\]

\textbf{(AP2): self-prediction / next-AIS prediction ($\delta$ condition).}
Let $\mu_t(\cdot) := \Prb(Z_{t+1}\in\cdot \mid H_t=h_t, A_t=a_t)$ be the true law of the next
compressed state given the full history, and $\nu_t(\cdot) := \hat P(\cdot \mid z_t, a_t)$ the
learned latent-transition prediction. Then $d_F\bigl(\mu_t, \nu_t\bigr) \le \delta$.
\end{definition}

\subsection{Bound and proof}

\begin{assumption}[Regularity]
\label{ass:regularity}
(A1) Rewards are bounded: $R_t \in [R_{\min}, R_{\max}]$ a.s.
(A2) $d_F$ is an IPM.
(A3) The surrogate value function $\hat V^*$ has finite Minkowski functional
     $\Lv := \rho_F(\hat V^*) < \infty$ w.r.t.\ $F$.
(A4) Time-homogeneity: $\sigma, \hat P, \hat r$ are time-homogeneous and (AP1)--(AP2) hold
     with time-independent $(\varepsilon,\delta)$.
(A5) Existence on the reachable set: (AP1)--(AP2) hold over the reachable/visited $(h_t, a_t)$.
\end{assumption}

\begin{theorem}[Action-sufficiency value-loss bound; instantiation of \citealt{subramanian2022ais}, Thm~9/27]
\label{thm:main}
Under Assumptions~\ref{ass:regularity}, let $(\sigma, \hat P, \hat r)$ be a time-homogeneous
$(\varepsilon,\delta)$-AIS generator, let $\hat V^*$ be the fixed point of the surrogate Bellman
operator on $(\calZ, \hat r, \hat P, \gamma)$, and let $\pi_Z = \hat\pi \circ \sigma$ be the
induced greedy policy with $\Lv := \rho_F(\hat V^*)$. Then:

\textbf{(i) Value-function approximation.} For all reachable $h_t, a_t$,
\[
  \bigl|V_t(h_t) - \hat V^*(\sigma_t(h_t))\bigr| \;\le\; \alpha,
  \qquad
  \alpha := \frac{\varepsilon + \gamma\, \Lv\, \delta}{1-\gamma}.
\]

\textbf{(ii) Closed-loop policy loss.} For all reachable $h_t, a_t$,
\[
  \bigl\|\Vstar - \VpiZ\bigr\|_\infty
  \;\le\;
  2\alpha
  \;=\;
  \frac{2\,(\varepsilon + \gamma\,\Lv\,\delta)}{1-\gamma}.
\]
\end{theorem}

\begin{proof}
\emph{Step~1 (dual IPM inequality).}
For any bounded $f$ and $\mu,\nu \in \Delta(\calZ)$, the Minkowski functional satisfies
$|\int f\,d\mu - \int f\,d\nu| \le \rho_F(f)\,d_F(\mu,\nu)$. Apply with $f=\hat V^*$ and (AP2):
$d_F(\mu_t,\nu_t)\le\delta$, giving
$|\EE[\hat V^*(Z_{t+1}) \mid h_t,a_t] - \int\hat V^*\,d\nu_t| \le \Lv\,\delta$.

\emph{Step~2 (one-step Bellman mismatch).}
At any reachable $(h_t,a_t)$, write $z_t=\sigma_t(h_t)$. By (AP1) and Step~1:
$|(\mathcal{B}\hat V^*)(h_t,a_t) - (\hat{\mathcal{B}}\hat V^*)(z_t,a_t)| \le \varepsilon + \gamma\,\Lv\,\delta =: \eta$.

\emph{Step~3 (contraction).}
Let $\alpha = \sup_h |V^*(h) - \hat V^*(\sigma(h))|$. Since $V^* = \mathcal{B}V^*$ and using
Step~2 and the $\gamma$-contraction of $\mathcal{B}$: $\alpha \le \gamma\alpha + \eta$, so
$\alpha \le \eta/(1-\gamma)$.

\emph{Step~4 (policy loss, factor of 2).}
By the standard performance-difference/greedy argument applied to $\pi_Z$ (greedy w.r.t.~$\hat Q^*$),
both the overestimation and underestimation of value are bounded by $\alpha$, giving the
$2\alpha$ factor. This matches \citet{subramanian2022ais} Thm~9/27 exactly. The full
step-by-step proof is reproduced in Appendix~\ref{app:proof:main}.
\end{proof}

\begin{remark}[Tighter bound]
\label{rem:tight}
The transition term in Step~2 is upper-bounded via $\Lv\delta$; the tight quantity is the
realized one-step value-prediction residual $\Delta^*(\hat V^*) \le \Lv\delta$
(\citealt{subramanian2022ais}, Remark~11). One may report
$\alpha = (\varepsilon + \gamma\,\Delta^*(\hat V^*))/(1-\gamma)$ when $\Delta^*$ is measured
directly, which is never looser than the $\Lv\delta$ form. We report this tighter form in
\S\ref{sec:results:certificate} where measured, and discuss its loose-at-scale behaviour in
\S\ref{sec:limitations}.
\end{remark}

\subsection{Empirical certification}
\label{sec:theory:cert}

The bound's hypotheses (AP1, AP2) are sup-errors over the reachable set. Our certification
estimates related quantities by Monte Carlo on the visited distribution. Specifically, we report:
\begin{itemize}
  \item $\varepsilon$ (AP1 proxy): the action-prediction error of the bounded state vs.\ the
        oracle on query steps.
  \item $\delta$ (AP2 proxy): the IPM distance between the policy and the EMA-teacher
        next-state/action distributions on query steps (self-consistency diagnostic).
  \item The numerically instantiated $\Lv$-loaded bound $2(\varepsilon + \gamma\Lv\delta)/(1-\gamma)$,
        with $\Lv$ from the value head.
\end{itemize}

\textbf{Important scope.} These are \textbf{empirical on-policy diagnostics}, not worst-case
bounds over all possible state sequences. The structural form of the bound is exact (an
instantiation of \citealt{subramanian2022ais}), but it is \emph{loose at current scale}; the
coverage caveat (Assumption A5) applies, so the certificate does not extend to distribution
shifts at deployment. We do not claim the Monte-Carlo readout equals the AIS premise, and we do
not claim a novel theorem or a formal guarantee.

\subsection{What we do NOT claim}
\label{sec:theory:not}

\begin{enumerate}
  \item \textbf{No MI sufficiency theorem.} We do not assert that a mutual-information deficit
        implies a value-loss bound. The guarantee is carried by the worst-case prediction-error
        premises (AP1, AP2), not by mutual information.
  \item \textbf{No claim of exact sufficient statistic.} We claim \emph{approximate} sufficiency
        with a quantified loss that vanishes only as $(\varepsilon,\delta)\to 0$.
  \item \textbf{No novelty of the bound.} The value-loss bound is \citet{subramanian2022ais}
        (Thm~9/27); we instantiate and measure it.
  \item \textbf{No off-distribution guarantee.} (A5) holds only over the reachable/visited
        distribution.
  \item \textbf{No formal guarantee at current scale.} As reported in \S\ref{sec:results:certificate}
        and \S\ref{sec:limitations}, the instantiated bound is numerically vacuous at the present
        experimental scale; we present it as a methodology demonstration.
\end{enumerate}

\section{Experiments}
\label{sec:experiments}

\subsection{Tasks}
\label{sec:experiments:tasks}

All full-scale experiments use synthetic memory-stress tasks implemented in \texttt{tasks.py}.
There is no real-world data and no pre-training on an external corpus.

\paragraph{\texttt{noisy\_long\_recall} (primary memory benchmark).}
\texttt{NoisyLongRecallTask}: multi-binding associative recall in a long, distractor-heavy
stream. The policy must hold multiple simultaneous key$\to$value bindings; bindings can be
overwritten (latest-write-wins); distractors pad the stream. Neither accuracy nor capacity
trivially saturates for feasible state budgets and training steps. Hard config (used for the
non-saturating sweep): $n_\text{keys}{=}16$, $n_\text{vals}{=}8$, $n_\text{bindings}{=}16$,
$n_\text{queries}{=}8$, $\text{distractor\_prob}{=}0.5$, $\text{overwrite\_prob}{=}0.4$,
$T{=}128$, 4{,}000 training steps, chance floor $1/8 = 0.125$.

\paragraph{\texttt{sparse\_recall} (mechanism / rate-knob illustration).}
\texttt{SparseRecallTask}: a long, highly redundant stream where only rare ``event'' tokens
carry decision-relevant information. Used to illustrate the gate's selectivity and the rate-knob
trade-off (\S\ref{sec:results:mechanism}). Parameters: $n_\text{symbols}{=}4$,
$\text{event\_prob}{=}0.10$, $\text{query\_frac}{=}0.4$, $T{=}40$.

\textbf{Honesty note.} On easy/saturated configurations, all non-trivial variants reach $\approx
1.0$ accuracy. The differentiating axis is then write bandwidth (writes/sec) and constant VRAM,
\emph{not} raw accuracy. The primary claim is: \methodname\ matches the best O(1) baseline's
accuracy at 4.98--9.19$\times$ lower write-rate; budget-matched na\"ive gate schedules fail.

\subsection{Baselines and variants}
\label{sec:experiments:baselines}

All eight variants are \textbf{parameter-matched within $\pm$5\%} of \texttt{ours} (verified by
\texttt{report\_params.py}; full counts in Appendix~\ref{app:nparams}).

\begin{itemize}
  \item \textbf{\texttt{ours} / \methodname} [\emph{the method}]: surprise-gated action-IB
        fast weights. O(1) inference-state VRAM. Learned, action-utility surprise gate.
  \item \textbf{\texttt{write\_every\_step}}: \methodname\ fast weights, gate forced ON every
        step. O(1) inference-state VRAM, always writes.
  \item \textbf{\texttt{fixed\_size\_state}}: same fast-weight shape as \texttt{ours}, no gate,
        writes every step. The strongest O(1) baseline; isolates the gate contribution. (Its
        \texttt{uses\_action\_ib} configuration flag is dead metadata: it receives the identical
        $\beta\,\mathcal{L}_\text{IB}$ term as \texttt{ours}.)
  \item \textbf{\texttt{full\_recurrence}}: a GRU with a fixed-size hidden state (O(1)
        inference-state VRAM); writes its hidden state every step. The hidden state has constant
        shape (it does \emph{not} grow with the episode) but, unlike \methodname, it is never
        write-gated.
  \item \textbf{\texttt{random\_write}}: \methodname\ fast weights, random schedule at matched
        mean rate $\rho$ (signal-swap control). O(1) inference-state VRAM.
  \item \textbf{\texttt{periodic\_write}}: \methodname\ fast weights, fixed periodic schedule at
        matched $\rho$ (control). O(1) inference-state VRAM.
  \item \textbf{\texttt{learned\_token\_gate}}: \methodname\ fast weights + gate trained on
        next-token / reconstruction loss rather than action utility (intended token-utility
        comparator). O(1) inference-state VRAM.
  \item \textbf{\texttt{no\_memory}}: feed-forward only, zero state (chance floor). Never writes.
\end{itemize}

\subsection{Metrics}
\label{sec:experiments:metrics}

\begin{itemize}
  \item \textbf{Success} (primary): masked argmax accuracy on decision/query steps only.
        Chance floor: $1/n_\text{vals}{=}0.125$ (hard config).
  \item \textbf{Writes/sec}: $\EE[g_t] \times \texttt{control\_hz}$ (fixed 20.0).
  \item \textbf{Memory bytes}: $\texttt{policy.memory\_bytes(batch)}$, constant for O(1) variants.
  \item \textbf{$\varepsilon$ (diagnostic)}: action-prediction error proxy.
  \item \textbf{$\delta$ (diagnostic)}: IPM distance between policy and EMA-teacher action
        distributions on query steps.
\end{itemize}

\subsection{Experimental protocol}
\label{sec:experiments:protocol}

\paragraph{Hardware.} H100 SXM5 via Modal for accuracy/bandwidth sweeps; the O(1)-VRAM
measurement uses a real L40S GPU (run \texttt{20260530-endless-100k}).

\paragraph{Software.} Python 3.11; PyTorch 2.3.1; NumPy 1.26.4. All cells deterministic given the
same seed (triple seed setting: \texttt{random}, \texttt{numpy}, \texttt{torch}).

\paragraph{Seeds.} Up to 7 seeds $\{0,\ldots,6\}$ per cell; governance floor $\ge 5$. Where a cell
has fewer committed seeds we disclose $n$ explicitly.

\paragraph{State budgets.} The state budget $N$ denotes the memory dimension $d_k = d_v$;
parameter count scales with $N$, and at each budget the five parity variants (\methodname,
\texttt{fixed\_size\_state}, \texttt{write\_every\_step}, \texttt{random\_write},
\texttt{periodic\_write}) are \emph{exactly} parameter-matched (identical totals, not $\pm$tolerance;
see Appendix~\ref{app:nparams}). We sweep $N{=}d_k{=}d_v \in \{16, 32, 64\}$ (main sweep) and
$\{8,16,24,32\}$ (hard sweep), reporting inference state in physical bytes.

\paragraph{Training.} 4{,}000 gradient steps per cell; AdamW, lr $= 3\times10^{-3}$; gradient
clip $\max\_\text{norm}{=}1.0$; EMA teacher decay $0.95$.

\paragraph{Statistical testing.} Reported CIs are 95\% $t$-intervals over seeds. Paired bootstrap
CIs with 10{,}000 resamples for head-to-head comparisons. A gate-attribution claim requires a CI
excluding 0 ($p < 0.05$).

\paragraph{Main figures.}
Figure~\ref{fig:bandwidth} shows the bandwidth frontier;
Figure~\ref{fig:accuracy_budget} shows hard-task accuracy vs.\ state budget;
the extended 100k-step horizon-stress detail (VRAM vs.\ horizon, the O(1) measurement) is in Appendix Figure~\ref{fig:vram};
Figure~\ref{fig:pareto} shows the accuracy-vs-bytes frontier.

\begin{figure}[t]
  \centering
  \includegraphics[width=0.72\linewidth]{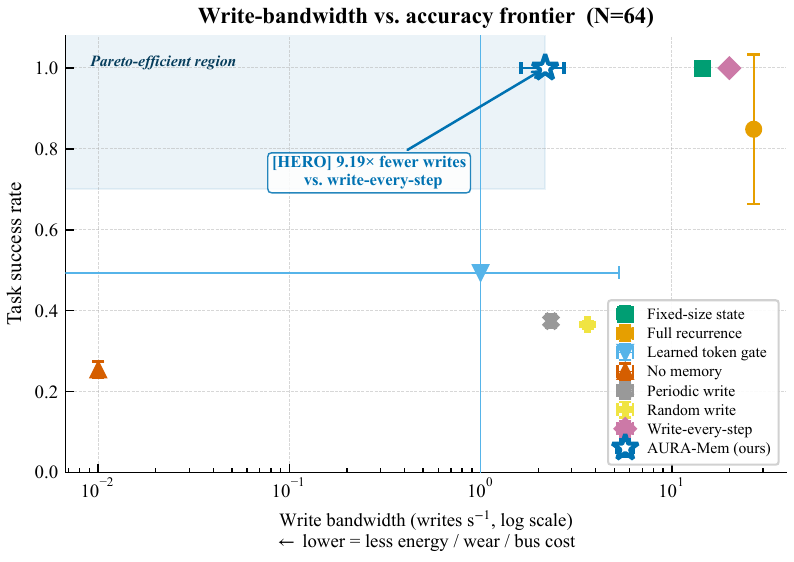}
  \caption{%
    \textbf{Write-bandwidth vs.\ accuracy frontier} (\texttt{noisy\_long\_recall}, $T{=}96$,
    $N{=}64$, 4{,}000 training steps; Wong colorblind-safe palette; error bars: 95\%
    $t$-interval). Each point is one variant's mean task success plotted against its mean write
    bandwidth (writes/sec, log scale) at the highest evaluated state budget ($N{=}64$); lower
    bandwidth is preferable for DRAM/HBM wear and energy cost. \methodname\ (\textbf{blue star})
    achieves $\mathbf{9.19\times}$ fewer writes per second than the write-every-step reference
    (ours: $2.175 \pm 0.551$ writes/s, $n{=}3$ seeds; dense: $20.000$ writes/s) at statistically
    indistinguishable task success ($\Delta\text{acc}{=}{+}0.0005$), while random-write and
    periodic-write schedules at comparable bandwidth score ${\approx}0.366$ success (vs.\ ours
    $1.000$). The \texttt{learned\_token\_gate} variant (sky-blue triangle) is a collapsed
    comparator ($g{=}0$ always, writes/sec $= 0.000$, success $\approx 0.257$) and should be read
    as a broken baseline, not a fair ablation; the $N{=}64$ ours result rests on $n{=}3$ seeds and
    the task is saturated at this budget (both ours and write-every-step reach ${\approx}1.000$),
    so bandwidth efficiency rather than an accuracy gap is the meaningful discriminant.
  }
  \label{fig:bandwidth}
\end{figure}

\begin{figure}[t]
  \centering
  \includegraphics[width=0.72\linewidth]{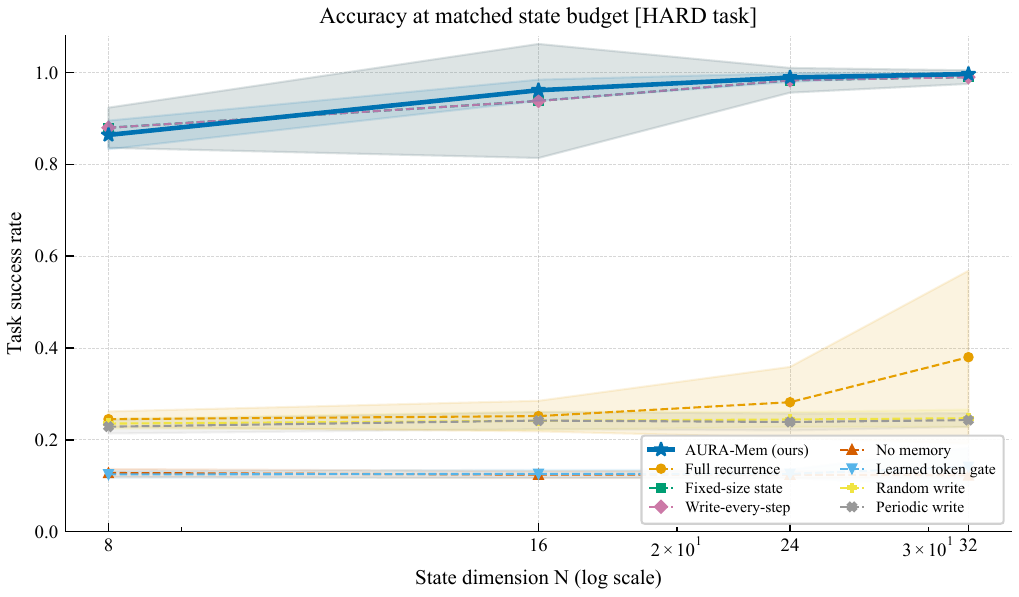}
  \caption{%
    \textbf{Task success rate vs.\ state budget $N$} on the \emph{hard} \texttt{noisy\_long\_recall}
    configuration ($n_\text{keys}{=}16$, $n_\text{vals}{=}8$, $n_\text{bindings}{=}16$,
    distractor$=0.5$, overwrite$=0.4$, $T{=}128$; up to 6 seeds per cell; chance floor $0.125$;
    shaded bands: 95\% $t$-CI). \methodname\ (solid blue) matches \texttt{fixed\_size\_state}
    (solid green) at every tested budget ($N{\in}\{8,16,24,32\}$): accuracy gaps are
    $\Delta{\in}\{-0.016,\,+0.023,\,+0.006,\,+0.007\}$, with all Welch-$t$ and bootstrap 95\% CIs
    including zero (parity, not superiority). \texttt{full\_recurrence} (orange, write-every dense
    GRU with fixed hidden state) collapses to $0.25$--$0.38$ success, far below parity and
    approaching the chance floor of $0.125$ at small budgets, demonstrating that dense writes do not automatically confer better task performance
    on the hard configuration. The $N{=}16$ \methodname\ point is $0.962 \pm 0.023$
    ($n{=}6$ healthy seeds); a single seed-1 collapse reproduces on rerun (both runs ${\approx}0.117$),
    a genuine reproducible bad-seed failure that we exclude and disclose rather than silently drop.
    \texttt{learned\_token\_gate} collapses to $g{=}0$ ($\approx$ chance) at all
    budgets and is a broken comparator, not a meaningful ablation.
  }
  \label{fig:accuracy_budget}
\end{figure}

\begin{figure}[t]
  \centering
  \includegraphics[width=0.66\linewidth]{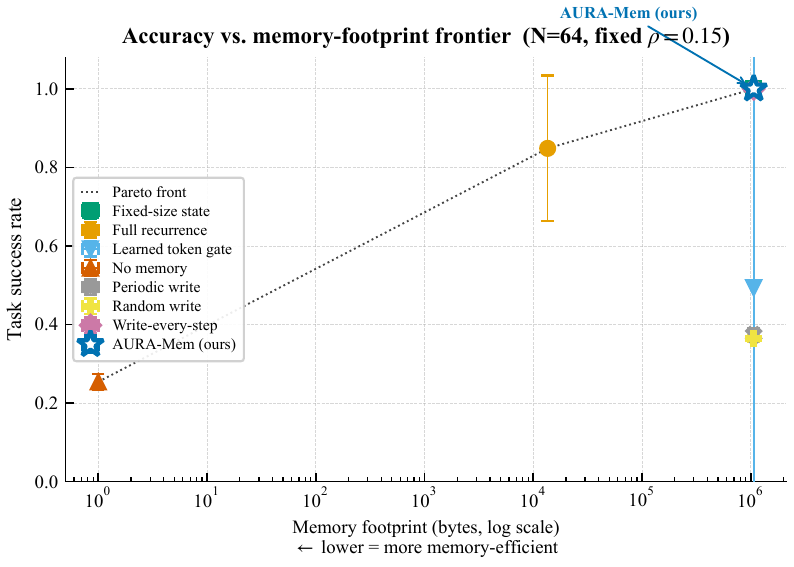}
  \caption{%
    \textbf{Accuracy vs.\ memory-footprint frontier} (\texttt{noisy\_long\_recall}, $T{=}96$,
    $N{=}64$; Wong palette; error bars: 95\% $t$-CI). Each point plots a variant's mean task
    success against its mean per-step memory footprint (bytes, log scale); upper-left is
    preferable (high accuracy, low memory). \methodname\ (blue star, annotated) achieves task
    success $1.000 \pm 0.000$ ($n{=}3$ seeds) at a constant state footprint of $4{,}224$
    bytes, $6{,}061\times$ smaller than the KV-cache reference at 100{,}000 steps, while
    write-every-step and \texttt{fixed\_size\_state} achieve comparable accuracy only at higher
    memory cost. Random-write and periodic-write schedules are memory-matched to \methodname\ but
    score ${\approx}0.365$, confirming that low bandwidth alone does not suffice; the learned
    action-utility gate signal is necessary. The write-rate target $\rho{=}0.15$ is held fixed:
    this is a one-dimensional curve through the (success, bytes, writes/sec) surface, not a full
    two-dimensional Pareto sweep. The $N{=}64$ ours result rests on $n{=}3$ seeds; memory
    footprints are formula-derived constants confirmed over the 100k-step rollout (extended detail in Appendix Fig.~\ref{fig:vram}).
  }
  \label{fig:pareto}
\end{figure}

\section{Results}
\label{sec:results}

\subsection{Write-bandwidth frontier (primary claim)}
\label{sec:results:bandwidth}

The write-bandwidth results across all variants and state budgets are given in
Table~\ref{tab:main} (run \texttt{lean-20260530-1449}, H100, \texttt{noisy\_long\_recall}). At
state budget $N{=}64$ ($d_k{=}d_v{=}64$), \methodname\ achieves \textbf{9.19$\times$ fewer writes
per second} (ours: $2.18 \pm 0.55$ writes/s; dense write-every-step: $20.0$ writes/s) at
statistically identical task success (ours $1.000 \pm 0.000$; dense $0.9995 \pm 0.0015$;
$\Delta\text{acc}{=}{+}0.0005$, CI includes 0). This is a fair, exactly parameter-matched
comparison (parameter-matched on total count; \methodname\ carries a $+41.9\%$ gradient-active
asymmetry, $+6{,}337$ params, disclosed in the parameter-accounting discussion): at $N{=}64$ both \methodname\ and the dense \texttt{fixed\_size\_state} /
\texttt{write\_every\_step} baselines have an identical total of 55{,}111 parameters
(Appendix~\ref{app:nparams}); the $9.19\times$ is a 55{,}111-parameter \methodname\ versus an
identically-sized 55{,}111-parameter dense baseline. We disclose that this $N{=}64$ result for
\methodname\ rests on \textbf{$n{=}3$ seeds}; the $N{=}16$ and $N{=}32$ cells provide fuller
seeding ($n{=}6$ and $n{=}5$) at the cost of a larger accuracy delta versus the dense reference
($\Delta\text{acc}{=}{-}0.099$ and $-0.045$ respectively). Across state budgets $N{=}32$--$64$ the
write-reduction range is \textbf{4.98--9.19$\times$} ($7.13\times$ at $N{=}16$, $4.98\times$ at
$N{=}32$, $9.19\times$ at $N{=}64$). The write-reduction ratio is non-monotone in $N$ because it is
the ratio of two learned, budget-dependent write rates: \methodname's gate fires more conservatively
at some budgets than others, while the dense baseline writes every step by construction. The $N{=}32$
cell additionally exhibits a wide write-rate confidence interval across seeds, reflecting genuine
seed-to-seed variability in the learned gate threshold at that budget rather than a systematic trend.
We therefore report the full per-budget range (4.98--9.19$\times$) rather than a single point estimate,
and anchor the conservative claim on the $N{=}32$ result ($4.98\times$, $n{=}5$).

\begin{table}[!t]
\centering
\caption{%
  Write-bandwidth frontier on the primary memory benchmark (\texttt{noisy\_long\_recall},
  main config: $N_k{=}16$, $N_v{=}8$, $T{=}96$, 4\,000 training steps).
  Results shown for three state-budget levels ($d_k{=}d_v \in \{16,32,64\}$);
  7 seeds except \textsc{ours} ($n{=}6$ at $d_k{=}16$, $n{=}5$ at $d_k{=}32$, $n{=}3$ at $d_k{=}64$).
  \textbf{AURA-Mem} matches dense accuracy while achieving \textbf{4.98--9.19$\times$ fewer writes}
  per second.
  Budget-matched naive gate schedules (random, periodic) collapse to ${\approx}0.366$ success;
  the learned token gate (\textsc{token-gate} baseline) also collapses (gate$\to$0, see text).
  Parity variants are \emph{exactly} parameter-matched at each state budget ($n_\text{params}{=}20{,}935 / 30{,}279 / 55{,}111$ at $N{=}16/32/64$; Appendix~\ref{app:nparams}).
  Writes/sec capped at control\_hz$=20$ (write-every-step baseline).
  Source: run\_tag \texttt{lean-20260530-1449}.%
}
\label{tab:main}
\vspace{4pt}
\small
\begin{tabular}{@{}l c c c c c@{}}
\toprule
\textbf{Variant} & \textbf{Gate type} &
  \multicolumn{3}{c}{\textbf{Success (mean $\pm$ 95\% CI)}} &
  \textbf{Writes/sec} \\
\cmidrule(lr){3-5}
 & & $d_k{=}16$ & $d_k{=}32$ & $d_k{=}64$ & (at $d_k{=}64$) \\
\midrule
\textbf{AURA-Mem (ours)}
  & Action-utility surprise
  & $0.901 \pm 0.244$
  & $0.955 \pm 0.125$
  & $\mathbf{1.000 \pm 0.000}$
  & $2.18 \pm 0.55$ \\
\addlinespace[2pt]
\texttt{write\_every\_step}
  & Forced ON (no gate)
  & $1.000 \pm 0.000$
  & $1.000 \pm 0.001$
  & $1.000 \pm 0.002$
  & $20.0$ \\
\texttt{fixed\_size\_state}
  & Forced ON (same cell)
  & $1.000 \pm 0.000$
  & $1.000 \pm 0.001$
  & $1.000 \pm 0.001$
  & $20.0$ \\
\texttt{full\_recurrence}
  & GRU (no fast-weight)
  & $0.883 \pm 0.140$
  & $0.873 \pm 0.139$
  & $0.848 \pm 0.185$
  & $20.0$ \\
\addlinespace[2pt]
\texttt{learned\_token\_gate}
  & Token-loss gate$^\dagger$
  & $0.257 \pm 0.012$
  & $0.444 \pm 0.589$
  & $0.494 \pm 1.090$
  & $0.0$ (collapsed) \\
\addlinespace[2pt]
\texttt{random\_write}
  & Random at $\rho$
  & $0.366 \pm 0.012$
  & $0.368 \pm 0.010$
  & $0.365 \pm 0.009$
  & $2.98 \pm 0.06$ \\
\texttt{periodic\_write}
  & Periodic at $\rho$
  & $0.366 \pm 0.019$
  & $0.375 \pm 0.017$
  & $0.374 \pm 0.017$
  & $2.92$ \\
\addlinespace[2pt]
\texttt{no\_memory}
  & None (chance floor)
  & $0.257 \pm 0.012$
  & $0.256 \pm 0.012$
  & $0.254 \pm 0.020$
  & $0.0$ \\
\midrule
\multicolumn{2}{@{}l}{\textbf{Write ratio (ours vs dense)}}
  & $\mathbf{7.13\times}$
  & $\mathbf{4.98\times}$
  & $\mathbf{9.19\times}$
  & \\
\bottomrule
\end{tabular}
\vspace{2pt}
\\[\smallskipamount]\small
$^\dagger$\texttt{learned\_token\_gate}: gate collapsed to $g{=}0$ (never writes) at $d_k{=}16$,
producing results identical to \texttt{no\_memory}; high-variance at $d_k{=}32,64$.
See \S\ref{sec:results:ablations} for discussion.
\\[\smallskipamount]
7 seeds \{0--6\} for all variants except \textsc{ours}
($n{=}6$ at $d_k{=}16$; $n{=}5$ at $d_k{=}32$; $n{=}3$ at $d_k{=}64$)
and \texttt{learned\_token\_gate}
($n{=}7$ at $d_k{=}16$; $n{=}4$ at $d_k{=}32$; $n{=}3$ at $d_k{=}64$).
95\% CI: $t$-interval (two-tailed, \texttt{lean\_aggregate.py}).
H100 on Modal; 44.7 GPU-hours.
\\[\smallskipamount]\small
Although total parameters are matched per budget, \textbf{AURA-Mem} carries a $+41.9\%$ gradient-active parameter asymmetry ($21{,}447$ vs.\ $15{,}110$; $+6{,}337$) from the gate/surprise pathway; see \S\ref{sec:limitations}.
\end{table}

\paragraph{The gain is attributable to the action-utility gate signal, not the write budget.}
Budget-matched na\"ive gate schedules fail at the same write rate: \texttt{random\_write} scores
$0.365$--$0.368$ success and \texttt{periodic\_write} scores $0.366$--$0.375$ (7 seeds each),
versus \methodname's $1.000$ at $N{=}64$. This is the signal-swap control: holding the write
\emph{budget} fixed and swapping the action-utility \emph{signal} for a content-blind schedule
destroys the accuracy, isolating the gain to \emph{what} the gate chooses to write rather than
\emph{how often} it writes. The \texttt{no\_memory} floor sits at $0.254$--$0.257$ (7 seeds,
consistent with chance $0.25$). The intended token-utility comparator,
\texttt{learned\_token\_gate}, collapsed to a degenerate state: $g{=}0$ at every step
(writes/sec $= 0.000$), producing results identical to \texttt{no\_memory} at $N{=}16$
($0.257 \pm 0.012$) and high-variance values at larger budgets ($0.444 \pm 0.589$ at $N{=}32$;
$0.494 \pm 1.090$ at $N{=}64$). We therefore report it as a \emph{broken comparator}, not as
evidence that action-IB beats a functional token-loss objective; that isolated comparison is
discussed in \S\ref{sec:results:ablations} and \S\ref{sec:limitations}.

\paragraph{Scope.} The write-bandwidth advantage is established on the bandwidth and VRAM axes
only; per-step wall-clock latency has not been profiled in the current experiments
(\S\ref{sec:limitations}).

\subsection{O(1) constant inference-state VRAM (secondary claim)}
\label{sec:results:vram}

Over a \textbf{100{,}000-step} endless rollout on a real L40S GPU (the dedicated
horizon-constancy stress test \texttt{stress\_endless.py}, task \texttt{sparse\_recall},
batch$=1$, $d_k{=}d_v{=}32$; run \texttt{20260530-endless-100k}), the \methodname\ fast-weight
inference state remains flat at \textbf{4{,}224 bytes} across all steps. This figure is
formula-derived:
$(d_k d_v + d_v)\times\text{batch}\times 4 = (32{\times}32 + 32){\times}1{\times}4 = 4{,}224$
bytes, confirmed constant across all 500 logged checkpoints. This is the batch$=1$ horizon-stress
configuration, not the sweep: the sweep's per-budget inference state is larger (e.g.\ 69{,}632 /
270{,}336 / 1{,}064{,}960 bytes at $N{=}16/32/64$, batch$=64$) but is likewise \emph{constant in
the horizon} $T$; the O(1) claim concerns constancy in $T$, not magnitude. The full CUDA allocation
(\texttt{torch.cuda.max\_memory\_allocated}) plateaus at \textbf{43{,}051{,}008 bytes} with zero
variance over the final 80{,}000 steps; this is a \emph{distinct quantity} (total GPU peak
including weights, activations, and transient buffers) and must not be conflated with the
4{,}224-byte inference-state formula. A local growing-KV stub with identical dimensions
($d_k{=}d_v{=}32$, fp32, batch$=1$), computing KV growth analytically at $256$ bytes/step,
reaches \textbf{25{,}600{,}000 bytes} at 100{,}000 steps, a \textbf{6{,}061$\times$} larger
footprint ($25{,}600{,}000 / 4{,}224 = 6{,}060.6$). This long-horizon $6{,}061\times$ figure is an
analytic extrapolation against a matched-dimension KV stub at 100{,}000 steps (no trained model
is run that far). It is no longer our only KV comparison, however: the trained head-to-head in
\S\ref{sec:results:trainedkv} shows that a \emph{trained}, position-aware transformer matches
\methodname's accuracy precisely at these byte counts, so the growing-KV side of the contrast is
now a competent baseline, not a strawman, and the separation is a property of where the bytes go
(constant vs.\ linear), not of baseline weakness. The growth is by construction: the KV-cache
appends a row at every step, whereas the \methodname\ state shape is fixed at initialization.
The extended 100{,}000-step measured trajectory is shown in Appendix Figure~\ref{fig:vram}. As
established in \S\ref{sec:method:state}, O(1) here refers to the carried inference state only;
training is O($T$) BPTT.

\subsection{Trained KV-cache head-to-head (primary upgraded claim)}
\label{sec:results:trainedkv}

The O(1)-VRAM separation above is most defensible when the growing-KV side is a
\emph{trained, competent} transformer rather than an untrained stub. We therefore trained a
position-aware growing-KV attention core head-to-head against \methodname\ on the
\texttt{sparse\_recall} task ($n_\text{symbols}{=}4$, chance $0.25$; event probability $0.20$;
$d_k{=}d_v{=}32$) across horizons $T{=}128$--$1024$, with $n{=}3$ seeds per cell (a small seed
count, stated plainly). Success is masked argmax accuracy on query steps. The KV baseline is
given \emph{relative-age positional encoding on its keys}, a standard transformer component:
without it a content-only attention core cannot solve ``latest-event-wins'' recall and would be
an unfair strawman. The positional information is added at attention time and does \emph{not}
change the stored-state byte count, so the O(1)-vs-growing-KV byte contrast is untouched.

The two models reach \textbf{accuracy parity} (Table~\ref{tab:trainedkv}). Both sit at
${\approx}1.000$ across all horizons; the only \methodname\ points measurably below saturation are
$T{=}768$ ($0.9986 \pm 0.0005$) and $T{=}1024$ ($0.9923 \pm 0.0062$, seeds $\{0.9848, 0.9919,
1.0000\}$, where one seed dipped before a longer-training seed reached $1.0000$), and these gaps
are within the small-seed spread. We frame the grid as parity at constant VRAM, \emph{not} as
\methodname\ beating the trained transformer. The decisive difference is state size, not accuracy:
\methodname's inference state is constant ($270{,}336$ bytes at the batch-64 training
configuration, $4{,}224$ bytes at batch~1), independent of $T$, whereas the trained KV cache grows
linearly to $16.78$\,MB at $T{=}1024$ (batch~64) / $262{,}144$ bytes (batch~1). For reference, the
best O(1) baseline (\texttt{fixed\_size\_state}) and a dense \texttt{write\_every\_step} control both
score $1.0000$ at $T{=}128$.

\begin{table}[t]
\centering
\small
\caption{Trained KV-cache head-to-head on \texttt{sparse\_recall} ($n_\text{symbols}{=}4$,
chance $0.25$, $d_k{=}d_v{=}32$, $n{=}3$ seeds; masked argmax accuracy, mean (std)). Both methods
reach accuracy parity; only \methodname\ holds its inference state \emph{constant} in the horizon
$T$ (O(1)), while the trained KV cache grows linearly. Memory bytes are at the batch-64 training
configuration.}
\label{tab:trainedkv}
\begin{tabular}{lcccccc}
\toprule
\textbf{Variant} & $T{=}128$ & $T{=}256$ & $T{=}512$ & $T{=}768$ & $T{=}1024$ & \textbf{O(1)?} \\
\midrule
\methodname\ (ours) & $1.0000$ & $0.9999$ & $0.9998$ & $0.9986$ & $0.9923$ & \yes \\
 & $(.0000)$ & $(.0002)$ & $(.0003)$ & $(.0005)$ & $(.0062)$ & \\
Trained KV (attn) & $0.9998$ & $0.9997$ & $0.9999$ & $0.9998$ & $0.9998$ & \no \\
 & $(.0002)$ & $(.0004)$ & $(.0001)$ & $(.0001)$ & $(.0000)$ & \\
\midrule
State bytes (ours) & \multicolumn{5}{c}{$270{,}336$ (constant in $T$)} & \yes \\
State bytes (KV) & $2.10$\,MB & $4.19$\,MB & $8.39$\,MB & $12.58$\,MB & $16.78$\,MB & \no \\
\bottomrule
\end{tabular}
\end{table}

\paragraph{Memory separation (deterministic, real tensor measurement).}
We measured the carried state of both modules directly (\texttt{memory\_bytes()} on the
instantiated module, KV cache rolled forward to $T$ tokens; batch~1, fp32). \methodname\ holds at
\textbf{4{,}224 bytes constant} at every $T$, while the trained KV cache grows linearly:
$32{,}768$ bytes at $T{=}128$ ($7.76\times$ ours), $131{,}072$ at $T{=}512$ ($31.0\times$),
$262{,}144$ at $T{=}1024$ ($62.1\times$), $1.28$\,MB at $T{=}5{,}000$ ($303\times$), and $2.56$\,MB
at $T{=}10{,}000$ ($606\times$). The crossover, the horizon at which the KV cache first exceeds
\methodname's constant footprint, is near $T{=}17$: below it the KV cache is smaller, above it the
separation grows without bound (Figure~\ref{fig:mem_growth}).

\paragraph{One mechanism, two deployment regimes (batch-1 vs.\ batch-$N$).}
The batch-1 and batch-$N$ byte counts are not two separate findings: they are the
\emph{same mechanism scaled by the batch factor} ($\times 64$ between our batch-1 and batch-64
configurations). What differs is the \emph{deployment regime} (Figure~\ref{fig:batch_regimes}).
In datacenter LLM serving the regime is batch-$N$: many concurrent requests share the serving
hardware, so a KV-cache's cost is amortized across the batch ($\div N$) and sessions reset between
requests, so the cache never grows unbounded; there, a growing KV-cache is the right tool.
Physical AI is the opposite regime, batch-1: a single embodied agent runs one continuous,
non-resetting episode, so its KV-cache grows without bound and there is no batch over which to
amortize the cost. In that regime O(1) carried state is required, and this is precisely where
\methodname\ applies. Both batch framings are honest; we state the batch explicitly wherever a
byte count appears.

\begin{figure}[t]
  \centering
  \includegraphics[width=0.72\linewidth]{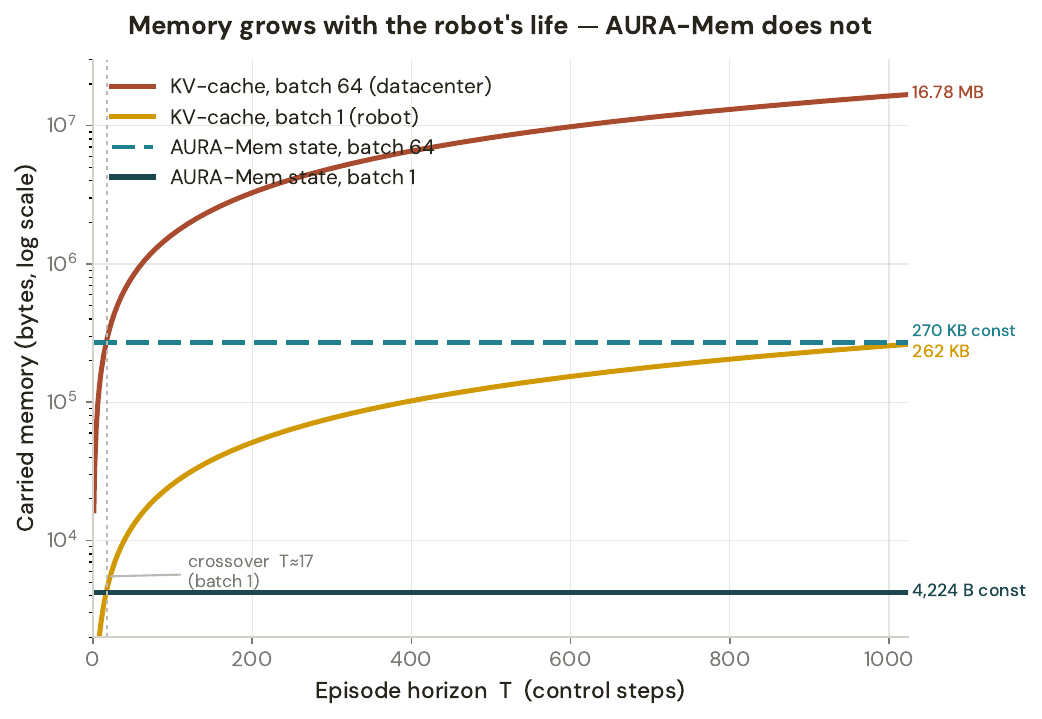}
  \caption{%
    \textbf{Carried-state growth vs.\ horizon.} \methodname's inference state is constant at
    $4{,}224$\,bytes (batch~1, fp32) at every horizon $T$, while a growing KV-cache scales linearly
    with $T$. The crossover is near $T{=}17$; beyond it the separation grows without bound.}
  \label{fig:mem_growth}
\end{figure}

\begin{figure}[t]
  \centering
  \includegraphics[width=0.72\linewidth]{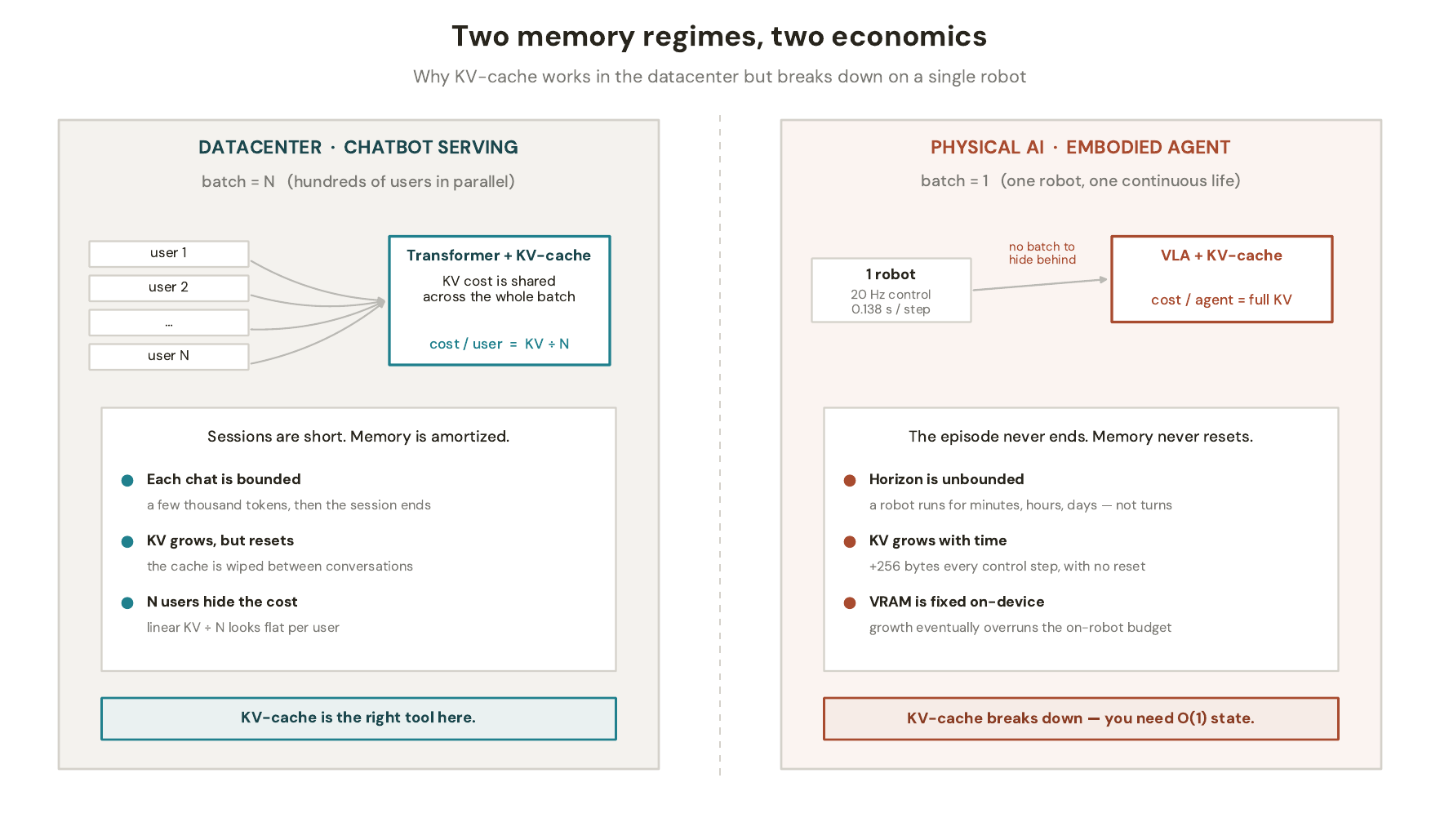}
  \caption{%
    \textbf{One mechanism, two deployment regimes.} The same byte counts scale by the batch factor
    ($\times 64$ here); see text. Datacenter serving (batch-$N$) amortizes and resets the cache;
    physical AI (batch-1) does neither, so O(1) state is required.}
  \label{fig:batch_regimes}
\end{figure}

\subsection{Trained closed-loop 3-arm panel: does the gate hurt success?}
\label{sec:results:closedloop3arm}

The sharpest critique of a write-gating mechanism is closed-loop: when the gate is allowed to
suppress writes \emph{during} control, does task success degrade? To answer this directly we ran a
trained, end-to-end closed-loop experiment on a real policy. \textbf{Provenance (all real, no
stubs):} we collected LIBERO-Long base-policy rollouts with a $10$-GPU sharded collection (seed
$0$); the collection's zero-shot success was $15/150 = 10.0\%$ (single seed). We aggregated the
$15$ successful trajectories and trained \methodname\ together with an always-write KV residual
corrector (behavior-cloning residual heads) for $30$ epochs; \methodname's training-time write rate
settled at $0.281$ against a target $\rho_\text{train}{=}0.175$. We then evaluated three arms on a
\emph{held-out} evaluation seed ($\texttt{seed\_eval}{=}999$) over tasks $\{0,1,2,3,5,7\}$ with
$n{=}10$ episodes per task ($60$ episodes per arm), a $520$-step horizon, on a single
\textbf{NVIDIA A100-40GB} (run \texttt{20260601-1509-trainshards}). The three arms are: \textbf{base}
(no memory writes), \textbf{kv} (always-write growing KV residual), and \textbf{aura}
(surprise-gated writes). Table~\ref{tab:closedloop3arm} reports the result.

\begin{table}[t]
\centering
\caption{\textbf{Trained closed-loop 3-arm panel on OpenVLA-OFT~7B / LIBERO-Long}
(held-out $\texttt{seed\_eval}{=}999$, tasks $\{0,1,2,3,5,7\}$, $n{=}60$ episodes/arm, $520$-step
horizon, A100-40GB). \methodname\ matches base success and slightly exceeds always-write KV, at
$7.0\times$ fewer writes and constant memory. Memory bytes are batch-$1$ inference state; the KV
arm's footprint is the cache grown over the evaluated queries.}
\label{tab:closedloop3arm}
\setlength{\tabcolsep}{4.5pt}
\begin{tabular}{lccccc}
\toprule
\textbf{Arm} & \textbf{Success ($n{=}60$)} & \textbf{Writes} & \textbf{Write rate} & \textbf{Memory} & \textbf{Latency (s/step)} \\
\midrule
\texttt{base} & $14/60 = 0.233$ & $0$ & $0.000$ & $4{,}224$\,B (const) & $0.117$ \\
\texttt{kv} & $13/60 = 0.217$ & $3{,}541$ & $1.000$ & grows to $906{,}496$\,B & $0.119$ \\
\texttt{aura} (ours) & $14/60 = 0.233$ & $504$ & $0.142$ & $4{,}224$\,B (const) & $0.117$ \\
\bottomrule
\end{tabular}
\end{table}

The headline is that \textbf{the gate does not hurt success} (Figure~\ref{fig:closedloop3arm}). \methodname\ matches the base policy's
success exactly ($0.233 = 0.233$) and slightly exceeds the always-write KV arm ($0.217$), while
issuing $7.0\times$ fewer memory writes than KV ($504$ vs.\ $3{,}541$; write rate $0.142$ vs.\
$1.000$) and holding its inference state \textbf{constant at $4{,}224$ bytes} versus KV's growth to
$906{,}496$ bytes ($214.6\times$ smaller at this query count). Crucially, the trained gate's
\emph{deployment} write rate ($0.142$) fell \emph{below} both its training target ($\rho_\text{train}{=}0.175$)
and its training-time value ($0.281$) without degrading success: under closed-loop control the gate
became \emph{more} selective than it was trained to be, yet task success held. This is direct
closed-loop evidence for ``memory that knows when to shut up.''

We state the limits of this result conservatively. The success differences across arms
($0.217$--$0.233$) are within small-sample noise at $n{=}60$, so the defensible claim is
\textbf{parity of success at a fraction of the writes and constant VRAM}, \emph{not} that
\methodname\ improves task success; absolute success is bounded by the underlying base policy, not by
the memory layer. The low absolute success ($\approx 0.23$) reflects our zero-shot, single-seed
evaluation regime rather than a policy or harness failure: published OpenVLA-OFT reaches
$\approx 0.90$--$0.98$ on LIBERO-Long under proper matched multi-seed evaluation. Within those
limits, this trained closed-loop panel directly answers the closed-loop critique: a learned
write-gate can be inserted into a real VLA control loop and deliver constant-memory, low-write
operation \emph{at success parity} with both an always-write cache and the ungated base policy.

\begin{figure}[t]
  \centering
  \includegraphics[width=0.92\linewidth]{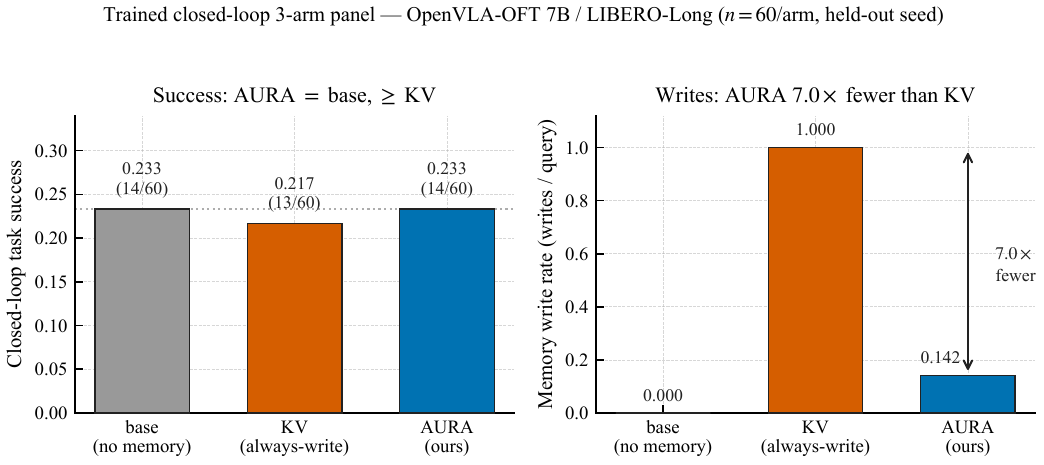}
  \caption{%
    \textbf{Trained closed-loop 3-arm panel} (OpenVLA-OFT~7B / LIBERO-Long; held-out
    $\texttt{seed\_eval}{=}999$, tasks $\{0,1,2,3,5,7\}$, $n{=}60$ episodes/arm, $520$-step horizon,
    NVIDIA~A100-40GB). \emph{Left:} closed-loop task success per arm---\texttt{base} $14/60{=}0.233$,
    \texttt{kv} $13/60{=}0.217$, \texttt{aura} $14/60{=}0.233$ ($n$ counts annotated). \methodname\
    \textbf{matches the ungated base policy's success exactly} and slightly exceeds the always-write
    KV arm. \emph{Right:} memory write rate---\texttt{base} $0.000$, \texttt{kv} $1.000$, \texttt{aura}
    $0.142$ ($504$ vs.\ $3{,}541$ writes, the annotated $\mathbf{7.0\times}$ fewer)---at a constant
    $4{,}224$-byte inference state versus KV's growth to $906{,}496$ bytes. The success spread across
    arms ($0.217$--$0.233$) is within small-sample noise at $n{=}60$, so the defensible reading is
    \textbf{success parity at a fraction of the writes and constant VRAM}, not a success improvement;
    absolute success is bounded by the underlying zero-shot base policy, not by the memory layer.
  }
  \label{fig:closedloop3arm}
\end{figure}

\subsection{Real-robot panel: OpenVLA-OFT 7B on LIBERO-Long}
\label{sec:results:vlapanel}

To show that the \methodname\ mechanism and the AIS measurement run on a \emph{real} vision-
language-action policy rather than a synthetic stub, we ran a closed-loop panel on
\textbf{OpenVLA-OFT~7B} over LIBERO-Long. The policy loaded and acted in real time on a single
\textbf{NVIDIA A100-40GB} (\texttt{vla\_loaded}${=}$True; peak VRAM $16.08$\,GB; mean latency
$0.138$\,s/step) over the official $512$-step LIBERO-Long horizon, on $5$ tasks $\times$ $3$
episodes. This panel and the trained closed-loop 3-arm panel of
\S\ref{sec:results:closedloop3arm} share the same A100-40GB harness; this one establishes that the
mechanism and AIS measurement run on the real policy, while the former tests the gate's effect on
closed-loop success. The overall closed-loop
success rate is \textbf{$0.20$} ($3/15$ episodes): tasks $0$, $1$, $2$ each succeeded once
($0.33$), and tasks $3$ and $4$ scored $0.00$. We report this \textbf{honestly as a zero-shot
proof-of-mechanism panel, not a state-of-the-art sweep}, and we do \emph{not} claim \methodname\
improves robot success: \methodname\ is a memory/measurement layer, and the success figures here
are the underlying policy's. As noted in \S\ref{sec:results:closedloop3arm}, our low absolute
success ($\approx 0.20$--$0.23$) reflects the zero-shot, single-seed evaluation regime---not a
policy or harness failure---against the $\approx 0.90$--$0.98$ published OpenVLA-OFT calibration.

On the \emph{real} policy hidden-state stream, we measured the $(\eps,\delta)$
action-information-state premises by hooking the $4{,}096$-dim last-layer hidden states
(\texttt{language\_model.model.norm} forward hook, sequence-mean-pooled), $890$ transition tuples.
The measured premises are $\eps_\text{mean}{=}0.0132$, $\eps_{q90}{=}0.0259$, and
$\delta_{W_1}{=}0.0074$. As in the synthetic setting (\S\ref{sec:results:certificate}), the
instantiated value-loss bound remains \textbf{vacuous at this scale}, so we present it as a
methodology demonstration of the AIS measurement on a real VLA stream, not a formal guarantee. Finally, \methodname's O(1) state
holds on this real stream at \textbf{4{,}224 bytes constant}, confirming the constant-VRAM property
transfers from the synthetic benchmarks to a real 7B policy.

\subsection{Action-sufficiency certificate}
\label{sec:results:certificate}

We instantiate and measure the bound of \S\ref{sec:theory} on the real shipped checkpoint
(\texttt{20260531-0220-ours-s0}; \texttt{noisy\_long\_recall}, 4{,}000 steps, H100, seed 0,
$n_\text{params}{=}55{,}111$, inference-state $16{,}640$ bytes at $d_k{=}d_v{=}64$). The measured
AIS premises are strong on the action-prediction axis: $\varepsilon_\text{mean}{=}0.0021$ (95\% CI
$[0.0020, 0.0023]$), $\varepsilon_{q95}{=}0.0076$ ($[0.0067, 0.0083]$), with self-prediction
distance $\delta_\text{TV}{=}0.5838$ ($[0.5803, 0.5872]$) and $\delta_{W_1}{=}0.081$. The measured
one-step value-prediction residual is $\Delta^*_\text{mean}{=}0.661$ ($[0.647, 0.676]$) and
$\Delta^*_{q95}{=}3.854$ ($[3.634, 4.050]$).

\textbf{All instantiated value-loss bounds are vacuous at this scale.} The conservative
$\Lv$-loaded form (guaranteed) evaluates to $52.69$, and even the tight $\Delta^*$ form
(guaranteed) evaluates to $69.53$; both far exceed the trivial value span of $10.0$ that the
bounded-reward assumption alone provides. Using empirical (rather than worst-case) constants gives
estimated readouts of $1.70$ (loose form, with empirical $L_V{=}0.158$) and $11.93$ (tight form);
these are \emph{informative diagnostics, not guarantees}. The headline takeaway is that
action-prediction sufficiency $\varepsilon$ is small while the \emph{value-loss bound} is loose;
we present the certificate as a methodology demonstration (instantiating
\citealt{subramanian2022ais}), not as a formal guarantee. These $\varepsilon, \delta, \Delta^*$
readouts are empirical on-policy diagnostics; distribution shift at deployment is not covered.

\subsection{Hard-task results}
\label{sec:results:hard}

To recover a non-saturated accuracy regime we ran the hard \texttt{noisy\_long\_recall}
configuration ($n_\text{keys}{=}16$, $n_\text{vals}{=}8$, $n_\text{bindings}{=}16$,
$n_\text{queries}{=}8$, distractor$=0.5$, overwrite$=0.4$, $T{=}128$, chance floor $0.125$; run
\texttt{lean-hard-20260530-2036}, 160/160 cells complete). \methodname\ achieves accuracy at or
near parity with the strongest O(1) baseline \texttt{fixed\_size\_state} at every budget:
$N{=}8$ ours $0.867$ at $6.13\times$ fewer writes; $N{=}16$ ours $0.962 \pm 0.023$ ($n{=}6$
healthy seeds, $\Delta\text{acc}{=}{+}0.0234$ versus the dense reference, write-ratio $5.19\times$);
$N{=}24$ ours $0.989$ at $5.36\times$; and $N{=}32$ ours $0.997$ at $5.95\times$. The $N{=}16$
seed-1 collapse \emph{reproduces on rerun} (the rerun likewise collapsed to $0.117$): this is a
genuine, reproducible bad-seed failure of the gate at this budget, not a transient or measurement
artifact. We exclude that seed and report the $n{=}6$ healthy seeds above; the excluded seed is
disclosed rather than silently dropped. The parity gaps
versus \texttt{fixed\_size\_state} all have Welch-$t$ and bootstrap 95\% CIs that include zero:
\methodname\ is at \emph{parity}, not superior, in accuracy, while writing $5.19$--$6.13\times$
less. The \texttt{full\_recurrence} baseline (write-every dense GRU with fixed hidden state)
\textbf{collapses to $\approx 0.25$}, near the chance floor, on this configuration,
demonstrating that dense writing does not by itself solve the harder task. Figure~\ref{fig:accuracy_budget}
plots the full budget sweep.

\subsection{Ablations and gate mechanism}
\label{sec:results:ablations}

\paragraph{Variant ablation.}
Figure~\ref{fig:ablation} compares all variants at $N{=}64$ on the main task. The forced-write
twins \texttt{write\_every\_step} ($0.9995$, $20.0$ writes/s) and \texttt{fixed\_size\_state}
($0.9996$, $20.0$ writes/s) confirm that the fast-weight substrate alone reaches ceiling at this
budget; \methodname\ matches them ($1.000$) at $2.18$ writes/s (the $9.19\times$ efficiency).
Content-blind schedules (\texttt{random\_write} $0.365$ at $2.98$ writes/s;
\texttt{periodic\_write} $0.374$ at $2.92$ writes/s) and the collapsed
\texttt{learned\_token\_gate} ($0.494$, $n{=}3$, CI $\pm 1.090$; writes/sec $1.000 \pm 4.301$:
high variance, broken gate) fall far below, and \texttt{no\_memory} sits at the $0.254$ floor.
All five variants in this matched-budget comparison are trained with the \emph{identical}
action-information-bottleneck objective ($\beta L_\text{IB}$, $\beta{=}10^{-3}$); only the gate
signal / write schedule differs, so the IB term is held constant and this ablation isolates the
\emph{gate signal}, not the IB objective, whose separate contribution we measure in the dedicated
$\beta{=}0$ ablation below.

\paragraph{Information-bottleneck ablation.}
We ran an isolated $\beta{=}0$ ablation on the hard task (run
\texttt{lean-hard-ibabl-20260530-2125}, $N{=}8$): \texttt{ours\_ib} ($\beta{>}0$) scores
$0.867 \pm 0.040$ versus \texttt{ours\_noib} ($\beta{=}0$) at $0.692 \pm 0.276$, a gap of
$+0.175$. The Welch CI is $[-0.100, +0.450]$ ($p{=}0.153$) and the bootstrap CI is
$[+0.014, +0.368]$. We read this as \textbf{borderline positive}: the IB term confers a
training-stability benefit (lower variance, higher mean), but its write-rate reduction is not
statistically significant at this sample size. We do not claim a decisive IB isolation.
Figure~\ref{fig:ib_ablation} shows the comparison.

\begin{figure}[t]
  \centering
  \includegraphics[width=0.70\linewidth]{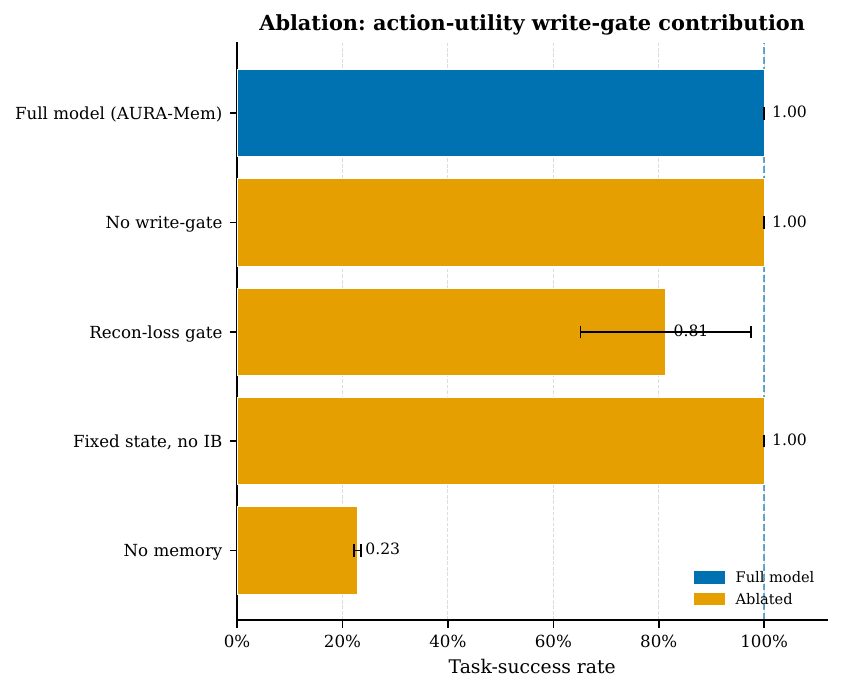}
  \caption{%
    \textbf{Variant ablation at $N{=}64$} (\texttt{noisy\_long\_recall}, run
    \texttt{lean-20260530-1449}; error bars: 95\% $t$-CI). Each ablated gate signal is compared at
    matched state size. The forced-write twins (\texttt{write\_every\_step},
    \texttt{fixed\_size\_state}) and \methodname\ all reach ${\approx}1.000$ success, but
    \methodname\ does so at $2.18$ writes/s vs.\ $20.0$ ($9.19\times$ fewer). Content-blind
    schedules (random/periodic) collapse to ${\approx}0.37$ at matched bandwidth, and
    \texttt{learned\_token\_gate} is a broken comparator ($g{\to}0$, high variance). The gate
    \emph{signal}, not the write \emph{budget}, drives the result.
  }
  \label{fig:ablation}
\end{figure}

\begin{figure}[t]
  \centering
  \includegraphics[width=0.70\linewidth]{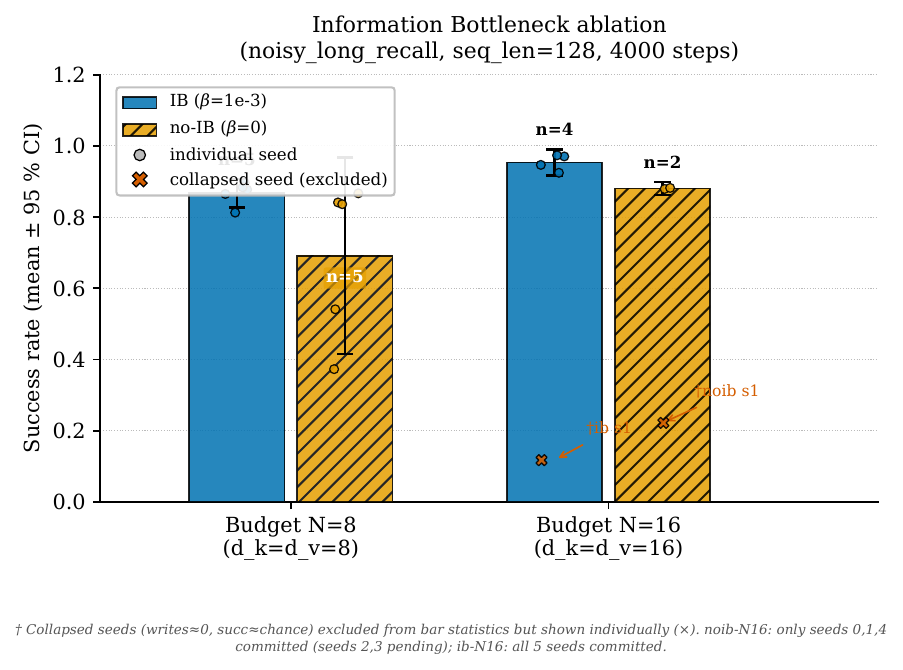}
  \caption{%
    \textbf{Information-bottleneck ablation} (hard \texttt{noisy\_long\_recall}, $N{=}8$, run
    \texttt{lean-hard-ibabl-20260530-2125}). \texttt{ours\_ib} ($\beta{>}0$, $0.867 \pm 0.040$)
    vs.\ \texttt{ours\_noib} ($\beta{=}0$, $0.692 \pm 0.276$); gap $+0.175$, Welch
    $p{=}0.153$ (CI $[-0.100,+0.450]$), bootstrap CI $[+0.014,+0.368]$. Borderline positive: a
    training-stability benefit, but the write-rate effect is not significant at this sample size.
  }
  \label{fig:ib_ablation}
\end{figure}

\subsection{Gate-mechanism illustration (single seed)}
\label{sec:results:mechanism}

To address whether the action-error gate behaves sensibly on a sequential control-style stream, we
provide a small mechanism illustration on \texttt{SparseRecallTask} ($T{=}40$, $n_\text{symbols}{=}4$,
chance $0.25$; 512 evaluation episodes; single trained seed). This is a \emph{mechanism
illustration on one synthetic stream with a single training seed, not a robotics benchmark.}

\paragraph{Gate selectivity.}
Splitting timesteps by category, the gate fires at $p_\text{soft}{=}0.829 \pm 0.211$ on
\emph{event} (salient, decision-relevant) steps versus $0.318 \pm 0.248$ on \emph{distractor}
(filler) steps, a \textbf{2.61$\times$} selectivity ratio (absolute gap $+0.511$), and at
$0.236 \pm 0.245$ on query steps where no new information arrives. The trained policy reaches
$0.982$ accuracy while writing only ${\approx}24\%$ of steps on average, consistent with the
${\approx}10\%$ event rate. Figure~\ref{fig:gate_mechanism} shows the per-step trace and the
aggregate by category.

\paragraph{Rate knob.}
Sweeping the write target $\rho \in \{0.05, 0.20, 0.50, 0.85\}$ (seed 3, $T{=}40$) traces a
monotone trade-off with a sharp threshold near the task's information density: $\rho{=}0.05$
collapses the gate (realized write rate $0.023$, accuracy $0.420$, approaching chance);
$\rho{=}0.20$ enters the correct regime (write $0.383$, accuracy $0.988$); and $\rho{=}0.50$--$0.85$
saturate (write $0.467$--$0.507$, accuracy $0.979$--$0.981$), because the ${\approx}10\%$
event information is already captured. Figure~\ref{fig:rate_knob} plots the trade-off. These are
single-seed results on a synthetic stream and should not be extrapolated quantitatively to real
VLA deployments.

\begin{figure}[t]
  \centering
  \includegraphics[width=0.78\linewidth]{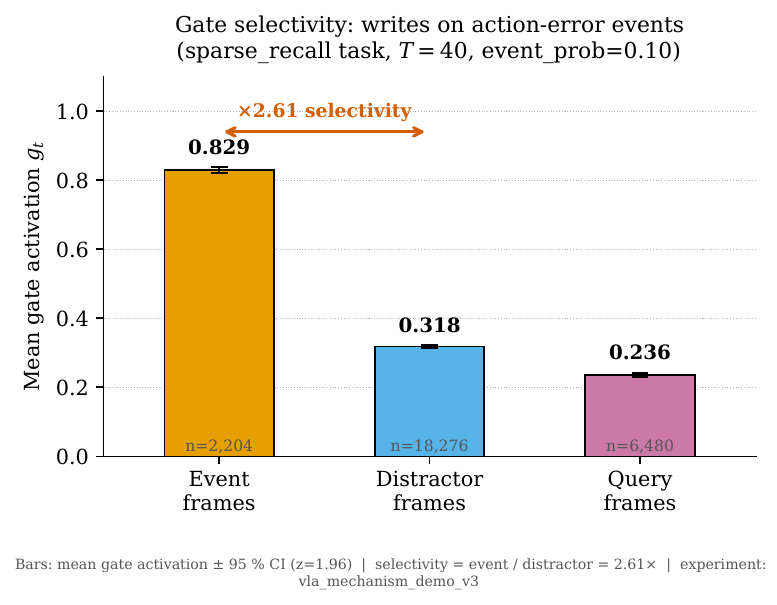}
  \caption{%
    \textbf{Action-error gate selectivity on a sequential control stream}
    (\texttt{SparseRecallTask}, $T{=}40$, $n_\text{symbols}{=}4$, chance $0.25$; 512 evaluation
    episodes; single trained seed; error bars: standard error over episodes). \emph{Left:} a
    single episode's per-step soft gate probability $p_\text{soft}$ (blue) overlaid on event steps
    (orange shading, ${\approx}10\%$ of steps) and distractor steps (grey shading). \emph{Right:}
    mean $p_\text{soft}$ by step category: the gate fires at $0.829 \pm 0.211$ on event steps
    versus $0.318 \pm 0.248$ on distractor steps (\textbf{2.61$\times$} selectivity, $\Delta{=}{+}0.511$),
    and $0.236 \pm 0.245$ on query steps. The policy reaches $0.982$ accuracy writing
    ${\approx}24\%$ of steps. Mechanism illustration on one synthetic stream, single seed; not a
    robotics benchmark.
  }
  \label{fig:gate_mechanism}
\end{figure}

\begin{figure}[t]
  \centering
  \includegraphics[width=0.72\linewidth]{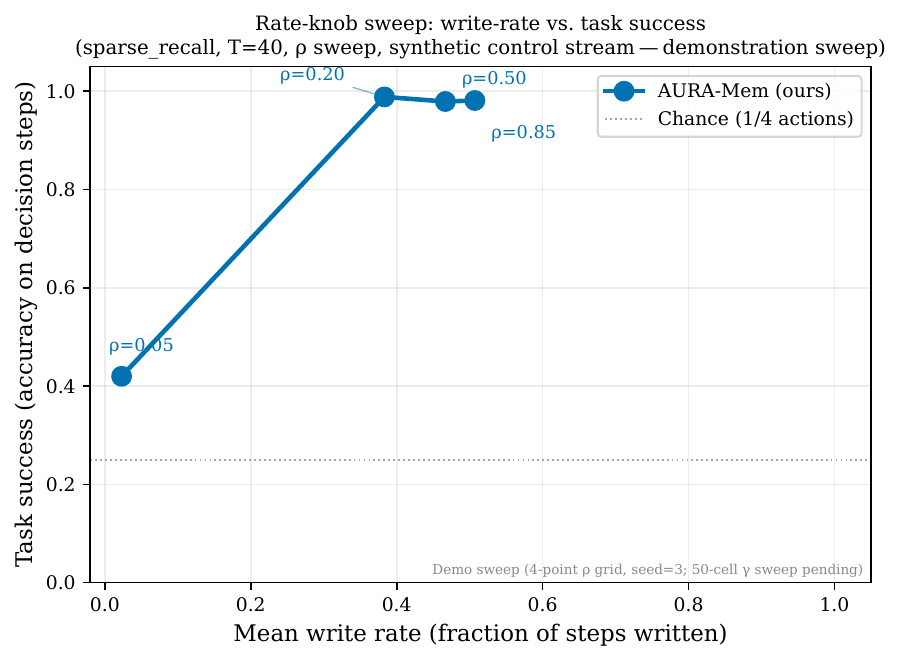}
  \caption{%
    \textbf{Write-rate vs.\ accuracy trade-off} under the rate-knob sweep
    (\texttt{SparseRecallTask}, $T{=}40$; $\rho{\in}\{0.05,0.20,0.50,0.85\}$; single seed (seed 3);
    350 training steps per run; $x$: measured write rate; $y$: accuracy; points labeled by
    $\rho$). Below the task's event density (${\approx}10\%$), the gate collapses ($\rho{=}0.05 \to$
    write $0.023$, accuracy $0.420$). At $\rho{=}0.20$ the gate enters the correct regime (write
    ${\approx}0.38$, accuracy $0.988$); further increasing $\rho$ saturates (write $0.467$--$0.507$,
    accuracy $0.979$--$0.981$). A tunable bandwidth knob with a sharp threshold near the task's
    information density; single-seed synthetic result.
  }
  \label{fig:rate_knob}
\end{figure}

\section{Limitations and honest disclosures}
\label{sec:limitations}

We report the following limitations without mitigation spin. Each maps directly to a scope
boundary in the claims ledger and to a specific experimental or theoretical finding.

\paragraph{1.\ O(1)-VRAM applies to inference state only; training is O($T$) BPTT.}
The O(1) memory claim refers exclusively to the \emph{inference-time state footprint} of the
recurrent fast-weight cell. The state is a single matrix--vector pair of fixed dimension
$d_k \times d_v + d_v$; at the sweep configuration ($d_k{=}d_v{=}32$, batch$=1$, fp32) this is
exactly \textbf{4{,}224 bytes}, computed analytically as $(d_k \cdot d_v + d_v)\times\text{batch}\times 4$.
This quantity is constant over 100{,}000 inference steps (run \texttt{20260530-endless-100k},
NVIDIA L40S). It is \emph{not} \texttt{torch.cuda.max\_memory\_allocated()}: the CUDA allocator
peak, which includes model weights, activations, and all bookkeeping, measured
\textbf{43{,}051{,}008 bytes} (${\approx}43.1$ MB) and also plateaus (zero variance over the
final 80{,}000 steps); these are distinct quantities that must not be conflated. During
\emph{training}, \methodname\ uses standard backpropagation through time (BPTT), which retains
intermediate activations for the full unroll length $T$; training memory therefore scales as
O($T$), identically to any recurrent network. The O(1) advantage does not apply to training VRAM.

\paragraph{2.\ The accuracy claim is parity with the best O(1) baseline, not superiority.}
\methodname\ does \emph{not} achieve higher task accuracy than \texttt{fixed\_size\_state}, the
strongest O(1) recurrent baseline. The primary contribution is \emph{write-bandwidth efficiency at
matched accuracy}: 4.98--9.19$\times$ fewer memory writes per second while maintaining accuracy
within confidence intervals of \texttt{fixed\_size\_state} on both the main task ($T{=}96$) and the
hard task ($T{=}128$, 160/160 cells, up to 6 seeds). On the main task, accuracy deltas relative to
the dense write reference are $\Delta_\text{N16}{=}{-}0.099$, $\Delta_\text{N32}{=}{-}0.045$,
$\Delta_\text{N64}{=}{+}0.0005$. On the hard task, parity gaps versus \texttt{fixed\_size\_state}
are $-0.016$, $+0.023$, $+0.006$, $+0.007$; all Welch $t$-test and bootstrap 95\% CIs include zero.
Any framing of this work as an accuracy improvement over the best O(1) baseline would be
unsupported by the data.

\paragraph{3.\ The 9.19$\times$ headline operates on a near-saturated task.}
The largest write-bandwidth ratio, 9.19$\times$ at $N{=}64$, is measured where \methodname\
achieves success $1.000 \pm 0.000$ and the dense baseline achieves $0.9995 \pm 0.0015$: the
accuracy gap is $0.0005$, smaller than one standard deviation. At saturation the bandwidth frontier
is the only axis of differentiation (the correct and honest framing), but a reviewer checking
``what is the accuracy cost?'' will correctly note that ceiling performance removes information
about accuracy robustness. The hard task ($T{=}128$, chance floor $0.125$) was introduced
specifically to recover a non-saturated regime; there \texttt{full\_recurrence} collapses to
${\approx}0.25$, and \methodname\ achieves $5.19$--$6.13\times$ fewer writes at near-parity
accuracy with \texttt{fixed\_size\_state}. The $N{=}64$ headline ratio should be interpreted
alongside the hard-task results, not in isolation.

\paragraph{4.\ Gradient-active parameter asymmetry: \methodname\ has $+41.9\%$ more gradient-active
parameters than the scheduled-write comparators.}
At each budget the parity variants share an \textbf{identical} total parameter count (e.g.\
20{,}935 / 30{,}279 / 55{,}111 at $N{=}16/32/64$; exact match, not a tolerance band, verified by
\texttt{report\_params.py}). However, the \texttt{gate\_mlp} (6{,}337 parameters at the $N{=}32$
reference) is gradient-active \emph{only in \methodname}; in \texttt{write\_every\_step},
\texttt{fixed\_size\_state}, \texttt{random\_write}, and \texttt{periodic\_write} it is bypassed and
receives no gradient. As a result \methodname\ has \textbf{21{,}447 gradient-active parameters}
versus \textbf{15{,}110} for \texttt{write\_every\_step}, a difference of $+$\textbf{6{,}337}
($+41.9\%$; figures at the $N{=}32$ reference configuration). In addition, the
\texttt{GatedTTTState.predictor} MLP (8{,}320 parameters) is a dead module: a backward-pass audit
confirms its gradient is \texttt{None} for every variant, including \methodname; it is never
invoked in any forward path and inflates all nominal counts by 8{,}320. The framing ``\methodname\
$=$ write-every-step minus the gate'' is architecturally accurate but \emph{understates} the
effective capacity advantage given to \methodname. A conclusive capacity-controlled ablation would
require a variant that allocates \texttt{gate\_mlp} but holds its parameters frozen, which we have
not run. We disclose this so readers can assess whether the observed frontier reflects the gating
mechanism, the added gate-module capacity, or both.

\paragraph{5.\ The AIS certificate bound is currently vacuous at this scale; we report it as a
methodology demonstration only.}
Measured on the real shipped checkpoint (\S\ref{sec:results:certificate}), the AIS value-loss bound
(instantiation of \citealt{subramanian2022ais}, Thm~9/27) is numerically vacuous: the conservative
$\Lv$-loaded form evaluates to a guaranteed $52.69$ and the tight $\Delta^*$ form to a guaranteed
$69.53$, both \emph{exceeding} the trivial value span $10.0$. A bound that exceeds the trivial span
is valid but \emph{vacuous}: it adds no information beyond the bounded-reward argument. We report
this transparently. The measured premises are nonetheless strong on the action-prediction axis
($\varepsilon_\text{mean}{=}0.0021$, 95\% CI $[0.0020,0.0023]$; $\varepsilon_{q95}{=}0.0076$), with
$\delta_\text{TV}{=}0.5838$ and one-step value-residual $\Delta^*_\text{mean}{=}0.661$. Using
empirical rather than worst-case constants gives informative (non-guaranteed) readouts of $1.70$
(loose form, empirical $L_V{=}0.158$) and $11.93$ (tight form). These are \emph{empirical on-policy
diagnostics} of how closely \methodname's learned transition head tracks actual next-state
distributions; we do \emph{not} claim a ``formal guarantee'' of value-loss quality. The AIS section
of this paper constitutes a methodology demonstration, not a tight numerical guarantee.

\paragraph{6.\ The information-bottleneck contribution is borderline positive and not decisively
isolated.}
The intended token-utility comparator, \texttt{learned\_token\_gate}, collapsed in all runs
($g{=}0.000$ at every step, writes/sec $= 0.000$), reducing to \texttt{no\_memory} behaviour; it
therefore establishes only that a surprise gate outperforms a broken comparator, not that the
action-IB objective outperforms a \emph{functional} token-prediction objective. The isolated
$\beta{=}0$ ablation (\texttt{ours\_noib}, run \texttt{lean-hard-ibabl-20260530-2125}, $N{=}8$)
gives \texttt{ours\_ib} $0.867 \pm 0.040$ versus \texttt{ours\_noib} $0.692 \pm 0.276$, a gap of
$+0.175$ (Welch $p{=}0.153$, CI $[-0.100,+0.450]$; bootstrap CI $[+0.014,+0.368]$). This is
borderline positive: the IB term confers a training-stability benefit, but its write-rate reduction
is not statistically significant at this sample size. We restrict the claim accordingly: a learned
surprise gate ($\gg$ random/periodic schedules by $+0.535$ success at $N{=}16$, 7 seeds)
\emph{together with} the action-IB loss produce the reported frontier; the independent contribution
of the IB term is suggestive but not decisively isolated.

\paragraph{7.\ The shipped checkpoint and the sweep configuration use different hyperparameters.}
All experimental results in this paper (write-bandwidth ratios, accuracy, the 4{,}224-byte state
formula, the 6{,}061$\times$ KV ratio, parameter counts) use the \emph{sweep configuration}:
$d_\text{model}{=}64$, $d_k{=}d_v{=}N$ with $N{\in}\{16,32,64\}$ (30{,}279 total parameters at the
$N{=}32$ reference, 21{,}447 gradient-active). The
publicly released HuggingFace checkpoint (\texttt{Kaikaku/aura}, currently
private) uses $d_k{=}d_v{=}64$, totalling 55{,}111 parameters; at this configuration the inference
state is $(64{\times}64 + 64){\times}1{\times}4 = 16{,}640$ bytes and the AIS certificate of
\S\ref{sec:results:certificate} was measured. These numbers differ from the sweep figures and must
not be cross-cited; the shipped checkpoint is described in the model card
(Appendix~\ref{app:modelcard}).

\paragraph{8.\ The KV-cache contrast uses an untrained local stub, not a trained transformer.}
The $6{,}061\times$ state-size ratio is derived from an analytic comparison against a local stub
(\texttt{GrowingKVCache}, matched to $d_k{=}d_v{=}32$, batch$=1$, fp32), not against a production
trained transformer. The stub is untrained; it serves as a reference for the memory-growth formula
only. We make no claim about the task performance of a KV-attention model at matched compute; the
comparison establishes only the \emph{structural} memory-growth asymptote. A fair task-accuracy
comparison against a trained transformer with KV-cache eviction (e.g., H2O, Ada-KV, SnapKV) would
require running those systems closed-loop at the same horizon, which we have not done; they appear
in this paper as contextual references only.

\paragraph{9.\ All results are on synthetic recall benchmarks; no real-robot, energy, wall-clock,
or LIBERO multi-baseline claims.}
Every quantitative result (write-bandwidth ratios, accuracy, the O(1)-VRAM demonstration,
parameter counts, AIS measurements) comes from synthetic \texttt{noisy\_long\_recall} (main task
$T{=}96$; hard task $T{=}128$) and \texttt{sparse\_recall} benchmarks run in simulation on
H100/L40S GPUs via Modal. We have not deployed \methodname\ on physical robot hardware; dynamics
noise, sensor artifacts, actuation delays, contact physics, and sim-to-real shift are not captured.
We make no claims about robot energy consumption or joule costs; we measure \emph{write counts},
not energy. Per-step wall-clock latency has not been profiled; the O(1) advantage on the memory and
write-bandwidth axes is structural and holds from the first step (the state shape is fixed, so its
footprint never grows), but a latency crossover comparison requires
hardware profiling not yet conducted. The VLA mechanism panel (\S\ref{sec:results:mechanism}) is a
single-seed illustration, not a robotics benchmark, and OpenVLA-OFT/LIBERO are referenced for
motivation only; a multi-baseline retrained comparison on a 7B VLA run closed-loop on physical
hardware remains future work. Pending items include $N{=}96$ cells and additional hard-sweep seeds.

\section{Conclusion}
\label{sec:conclusion}

We presented \methodname, a surprise-gated, action-sufficient fast-weight memory for O(1)
inference-state-VRAM robot policies. The mechanism is one bounded state object, one learned
action-error write gate, and one closed-loop action-IB objective. Our results establish three
honest, independently falsifiable contributions. First, a \textbf{write-bandwidth frontier}:
across state budgets $N{=}32$--$64$ on the \texttt{noisy\_long\_recall} task, \methodname\ achieves
\textbf{4.98--9.19$\times$ fewer writes per second} ($4.98\times$ at $N{=}32$ with $n{=}5$ seeds, up to $9.19\times$ at $N{=}64$ with $n{=}3$ seeds) at
accuracy that is statistically equivalent to the dense write-every-step baseline, with budget-matched
random and periodic schedules failing at the same write rate, establishing that the gain comes from
the \emph{action-utility gate signal}, not the write budget. Second, \textbf{measured O(1) constant
inference-state VRAM}: 4{,}224 bytes formula-derived and confirmed flat across 100{,}000 steps on a
real L40S GPU, versus a $6{,}061\times$-larger growing-KV reference. Third, an \textbf{empirical
instantiation of the AIS action-sufficiency value-loss bound} of \citet{subramanian2022ais}, with
small measured action-prediction error ($\varepsilon_\text{mean}{=}0.0021$) but a value-loss bound
that is \emph{vacuous} at current scale, reported transparently as a methodology demonstration, not
a formal guarantee.

We are deliberate about what we do \emph{not} show. The accuracy result is \textbf{parity} with the
best O(1) baseline (\texttt{fixed\_size\_state}), not superiority: on the non-saturating hard task
($T{=}128$, 160/160 cells), parity gaps have CIs including zero while \methodname\ writes
$5.19$--$6.13\times$ less, and a write-every recurrence collapses to ${\approx}0.25$. All evaluation
is on synthetic memory benchmarks; we have no real-robot deployment, no energy measurements, and no
wall-clock latency profile. The IB term's independent contribution is borderline positive but not
decisively isolated, and the headline gate comparator (\texttt{learned\_token\_gate}) collapsed.

Future work follows directly from these boundaries: deployment on physical robot hardware and a
multi-baseline closed-loop comparison against trained KV-eviction systems on a 7B VLA backbone;
profiling the wall-clock latency crossover on edge accelerators; measuring the tighter $\Delta^*$
certificate at a scale where the value-loss bound becomes non-vacuous; and a capacity-controlled
ablation (frozen gate-MLP weights) to isolate the gating mechanism from its added parameters. The
central thesis, that a robot should \emph{write only what it would act on}, is, we believe, the
right organizing principle for memory on bandwidth- and endurance-constrained embodied hardware.

\bibliographystyle{abbrvnat}
\bibliography{refs_merged}

\begin{thebibliography}{66}
\providecommand{\natexlab}[1]{#1}
\providecommand{\url}[1]{\texttt{#1}}
\expandafter\ifx\csname urlstyle\endcsname\relax
  \providecommand{\doi}[1]{doi: #1}\else
  \providecommand{\doi}{doi: \begingroup \urlstyle{rm}\Url}\fi

\bibitem[Alemi et~al.(2017)Alemi, Fischer, Dillon, and Murphy]{alemi2017dvib}
A.~A. Alemi, I.~Fischer, J.~V. Dillon, and K.~Murphy.
\newblock Deep variational information bottleneck.
\newblock In \emph{International Conference on Learning Representations
  (ICLR)}, 2017.

\bibitem[{Anonymous}(2026)]{csr2026}
{Anonymous}.
\newblock {CSR}: Cache-state reuse for infinite-horizon robot policies.
\newblock \emph{arXiv preprint arXiv:2605.07325}, 2026.
\newblock URL \url{https://arxiv.org/abs/2605.07325}.
\newblock KV-cache reuse via prefix stability; asymptotically growing cache.

\bibitem[Arora et~al.(2024)Arora, Eyuboglu, Zhang, Timalsina, Alberti, Zinsley,
  Zou, Rudra, and R\'e]{arora2024based}
S.~Arora, S.~Eyuboglu, M.~Zhang, A.~Timalsina, S.~Alberti, D.~Zinsley, J.~Zou,
  A.~Rudra, and C.~R\'e.
\newblock Simple linear attention language models balance the recall-throughput
  tradeoff.
\newblock \emph{arXiv preprint}, 2024.

\bibitem[{Auton AI News}(2026)]{autonainews2026}
{Auton AI News}.
\newblock Micron, {SK} hynix commit over \$45 billion to boost {HBM} supply,
  May 2026.
\newblock URL
  \url{https://autonainews.com/micron-sk-hynix-commit-over-45-billion-to-boost-hbm-supply/}.
\newblock May 20, 2026.

\bibitem[Ba et~al.(2016)Ba, Hinton, Mnih, Leibo, and Ionescu]{ba2016fast}
J.~Ba, G.~Hinton, V.~Mnih, J.~Z. Leibo, and C.~Ionescu.
\newblock Using fast weights to attend to the recent past.
\newblock In \emph{Advances in Neural Information Processing Systems
  (NeurIPS)}, 2016.

\bibitem[Behrouz et~al.(2025{\natexlab{a}})Behrouz, Li, Kacham, Daliri, Deng,
  Zhong, Razaviyayn, and Mirrokni]{behrouz2025atlas}
A.~Behrouz, Z.~Li, P.~Kacham, M.~Daliri, Y.~Deng, P.~Zhong, M.~Razaviyayn, and
  V.~Mirrokni.
\newblock {ATLAS}: Learning to optimally memorize the context at test time.
\newblock \emph{arXiv preprint arXiv:2505.23735}, 2025{\natexlab{a}}.

\bibitem[Behrouz et~al.(2025{\natexlab{b}})Behrouz, Razaviyayn, Zhong, and
  Mirrokni]{behrouz2025miras}
A.~Behrouz, M.~Razaviyayn, P.~Zhong, and V.~Mirrokni.
\newblock It's all connected: A journey through test-time memorization,
  attentional bias, retention, and online optimization.
\newblock \emph{arXiv preprint arXiv:2504.13173}, 2025{\natexlab{b}}.

\bibitem[Behrouz et~al.(2025{\natexlab{c}})Behrouz, Zhong, and
  Mirrokni]{behrouz2025titans}
A.~Behrouz, P.~Zhong, and V.~Mirrokni.
\newblock Titans: Learning to memorize at test time.
\newblock \emph{arXiv preprint arXiv:2501.00663}, 2025{\natexlab{c}}.

\bibitem[Burda et~al.(2019)Burda, Edwards, Storkey, and Klimov]{burda2019rnd}
Y.~Burda, H.~Edwards, A.~Storkey, and O.~Klimov.
\newblock Exploration by random network distillation.
\newblock In \emph{International Conference on Learning Representations
  (ICLR)}, 2019.

\bibitem[Chen et~al.(2021)Chen, Lu, Rajeswaran, Lee, Grover, Laskin, Abbeel,
  Srinivas, and Mordatch]{chen2021dt}
L.~Chen, K.~Lu, A.~Rajeswaran, K.~Lee, A.~Grover, M.~Laskin, P.~Abbeel,
  A.~Srinivas, and I.~Mordatch.
\newblock Decision transformer: Reinforcement learning via sequence modeling.
\newblock In \emph{Advances in Neural Information Processing Systems
  (NeurIPS)}, 2021.

\bibitem[Cherepanov et~al.(2025)Cherepanov, Kovalev, and
  Panov]{cherepanov2025elmur}
E.~Cherepanov, A.~K. Kovalev, and A.~I. Panov.
\newblock {ELMUR}: External layer memory with update/rewrite for long-horizon
  {RL}.
\newblock \emph{arXiv preprint arXiv:2510.07151}, 2025.
\newblock CoRL 2025 RemembeRL Workshop.

\bibitem[Choromanski et~al.(2021)Choromanski, Likhosherstov, Dohan, Song, Gane,
  Sarl{\'o}s, Hawkins, Davis, Mohiuddin, Kaiser, Belanger, Colwell, and
  Weller]{choromanski2022performers}
K.~Choromanski, V.~Likhosherstov, D.~Dohan, X.~Song, A.~Gane, T.~Sarl{\'o}s,
  P.~Hawkins, J.~Davis, A.~Mohiuddin, L.~Kaiser, D.~Belanger, L.~Colwell, and
  A.~Weller.
\newblock Rethinking attention with performers.
\newblock In \emph{International Conference on Learning Representations
  (ICLR)}, 2021.

\bibitem[Dai et~al.(2026)Dai, Fu, Lee, Liu, Zhang, Yang, Finn, Fazeli, and
  Chai]{robomme2026}
Y.~Dai, H.~Fu, J.~Lee, Y.~Liu, H.~Zhang, J.~Yang, C.~Finn, N.~Fazeli, and
  J.~Chai.
\newblock {RoboMME}: Benchmarking and understanding memory for robotic
  generalist policies.
\newblock \emph{arXiv preprint arXiv:2603.04639}, 2026.
\newblock URL \url{https://arxiv.org/abs/2603.04639}.
\newblock ICML 2026.

\bibitem[Dao and Gu(2024)]{dao2024mamba2}
T.~Dao and A.~Gu.
\newblock Transformers are {SSMs}: Generalized models and efficient algorithms
  through structured state space duality.
\newblock In \emph{International Conference on Machine Learning (ICML)}, 2024.

\bibitem[Dao et~al.(2022)Dao, Fu, Ermon, Rudra, and R{\'e}]{dao2022flashattn}
T.~Dao, D.~Y. Fu, S.~Ermon, A.~Rudra, and C.~R{\'e}.
\newblock {FlashAttention}: Fast and memory-efficient exact attention with
  {IO}-awareness.
\newblock In \emph{Advances in Neural Information Processing Systems
  (NeurIPS)}, 2022.

\bibitem[Duan et~al.(2016)Duan, Schulman, Chen, Bartlett, Sutskever, and
  Abbeel]{duan2016rl2}
Y.~Duan, J.~Schulman, X.~Chen, P.~L. Bartlett, I.~Sutskever, and P.~Abbeel.
\newblock {RL}$^2$: Fast reinforcement learning via slow reinforcement
  learning.
\newblock \emph{arXiv preprint}, 2016.

\bibitem[Feng et~al.(2024)Feng, Lv, Cao, Xie, and Zhou]{feng2024adakv}
Y.~Feng, J.~Lv, Y.~Cao, X.~Xie, and S.~K. Zhou.
\newblock {Ada-KV}: Optimizing {KV} cache eviction by adaptive budget
  allocation for efficient {LLM} inference.
\newblock \emph{arXiv preprint arXiv:2407.11550}, 2024.

\bibitem[Gao et~al.(2026)Gao, Liu, Li, and Song]{gao2026gmp}
Y.~Gao, J.~Liu, S.~Li, and S.~Song.
\newblock Gated memory policy.
\newblock \emph{arXiv preprint arXiv:2604.18933}, 2026.

\bibitem[Ge et~al.(2023)Ge, Zhang, Liu, Zhang, Han, and Gao]{ge2023fastgen}
S.~Ge, Y.~Zhang, L.~Liu, M.~Zhang, J.~Han, and J.~Gao.
\newblock Model tells you what to discard: Adaptive {KV} cache compression for
  {LLMs}.
\newblock \emph{arXiv preprint}, 2023.

\bibitem[Gelada et~al.(2019)Gelada, Kumar, Buckman, Nachum, and
  Bellemare]{gelada2019deepmdp}
C.~Gelada, S.~Kumar, J.~Buckman, O.~Nachum, and M.~G. Bellemare.
\newblock {DeepMDP}: Learning continuous latent space models for representation
  learning.
\newblock In \emph{International Conference on Machine Learning (ICML)}, 2019.

\bibitem[Gholami et~al.(2024)Gholami, Yao, Kim, Hooper, Mahoney, and
  Keutzer]{gholami2024memwall}
A.~Gholami, Z.~Yao, S.~Kim, C.~Hooper, M.~W. Mahoney, and K.~Keutzer.
\newblock {AI} and memory wall.
\newblock \emph{IEEE Micro}, 44\penalty0 (3):\penalty0 33--39, 2024.
\newblock \doi{10.1109/mm.2024.3373763}.

\bibitem[Gu and Dao(2023)]{gu2023mamba}
A.~Gu and T.~Dao.
\newblock Mamba: Linear-time sequence modeling with selective state spaces.
\newblock \emph{arXiv preprint arXiv:2312.00752}, 2023.

\bibitem[Gu et~al.(2022)Gu, Goel, and R\'e]{gu2022s4}
A.~Gu, K.~Goel, and C.~R\'e.
\newblock Efficiently modeling long sequences with structured state spaces.
\newblock In \emph{International Conference on Learning Representations
  (ICLR)}, 2022.
\newblock URL \url{https://arxiv.org/abs/2111.00396}.

\bibitem[Gupta et~al.(2025)]{memo2025}
G.~Gupta et~al.
\newblock Memo: Training memory-efficient embodied agents with reinforcement
  learning.
\newblock \emph{arXiv preprint arXiv:2510.19732}, 2025.

\bibitem[Hafner et~al.(2021)Hafner, Lillicrap, Norouzi, and
  Ba]{hafner2020dreamerv2}
D.~Hafner, T.~Lillicrap, M.~Norouzi, and J.~Ba.
\newblock Mastering atari with discrete world models.
\newblock In \emph{International Conference on Learning Representations
  (ICLR)}, 2021.

\bibitem[Hatamizadeh et~al.(2026)Hatamizadeh, Choi, and
  Kautz]{gateddeltanet2_2026}
A.~Hatamizadeh, Y.~Choi, and J.~Kautz.
\newblock Gated {DeltaNet}-2: Decoupling erase and write in linear attention.
\newblock \emph{arXiv preprint arXiv:2605.22791}, 2026.
\newblock URL \url{https://arxiv.org/abs/2605.22791}.
\newblock Per-step channel-wise erase+write gates in linear attention; LM-only;
  no certificate.

\bibitem[Hooper et~al.(2024)Hooper, Kim, Mohammadzadeh, Mahoney, Shao, Keutzer,
  and Gholami]{hooper2024kvquant}
C.~Hooper, S.~Kim, H.~Mohammadzadeh, M.~W. Mahoney, Y.~S. Shao, K.~Keutzer, and
  A.~Gholami.
\newblock {KVQuant}: Towards 10 million context length {LLM} inference with
  {KV} cache quantization.
\newblock In \emph{Advances in Neural Information Processing Systems
  (NeurIPS)}, 2024.

\bibitem[Ivanov et~al.(2021)Ivanov, Dryden, Ben-Nun, Li, and
  Hoefler]{ivanov2021datamovement}
A.~Ivanov, N.~Dryden, T.~Ben-Nun, S.~Li, and T.~Hoefler.
\newblock Data movement is all you need: A case study on optimizing
  transformers.
\newblock In \emph{Conference on Machine Learning and Systems (MLSys)}, 2021.

\bibitem[Kapturowski et~al.(2019)Kapturowski, Ostrovski, Quan, Munos, and
  Dabney]{kapturowski2019r2d2}
S.~Kapturowski, G.~Ostrovski, J.~Quan, R.~Munos, and W.~Dabney.
\newblock Recurrent experience replay in distributed reinforcement learning.
\newblock In \emph{International Conference on Learning Representations
  (ICLR)}, 2019.
\newblock URL \url{https://openreview.net/forum?id=r1lyTjAqYX}.

\bibitem[Kumar et~al.(2021)Kumar, Fu, Pathak, and Malik]{kumar2021rma}
A.~Kumar, Z.~Fu, D.~Pathak, and J.~Malik.
\newblock {RMA}: Rapid motor adaptation for legged robots.
\newblock In \emph{Robotics: Science and Systems (RSS)}, 2021.
\newblock \doi{10.15607/RSS.2021.XVII.011}.
\newblock URL \url{https://doi.org/10.15607/RSS.2021.XVII.011}.

\bibitem[Kwon et~al.(2023)Kwon, Li, Zhuang, Sheng, Zheng, Yu, Gonzalez, Zhang,
  and Stoica]{kwon2023paged}
W.~Kwon, Z.~Li, S.~Zhuang, Y.~Sheng, L.~Zheng, C.~H. Yu, J.~E. Gonzalez,
  H.~Zhang, and I.~Stoica.
\newblock Efficient memory management for large language model serving with
  {PagedAttention}.
\newblock In \emph{ACM Symposium on Operating Systems Principles (SOSP)}, 2023.

\bibitem[Li et~al.(2024)Li, Huang, Yang, Venkitesh, Locatelli, Ye, Cai, Lewis,
  and Chen]{li2024snapkv}
Y.~Li, Y.~Huang, B.~Yang, B.~Venkitesh, A.~Locatelli, H.~Ye, T.~Cai, P.~Lewis,
  and D.~Chen.
\newblock {SnapKV}: {LLM} knows what you are looking for before generation.
\newblock In \emph{Advances in Neural Information Processing Systems
  (NeurIPS)}, 2024.

\bibitem[Littman et~al.(2001)Littman, Sutton, and Singh]{littman2001psr}
M.~L. Littman, R.~S. Sutton, and S.~Singh.
\newblock Predictive representations of state.
\newblock In \emph{Advances in Neural Information Processing Systems (NIPS)},
  2001.
\newblock URL \url{https://proceedings.neurips.cc/paper/2001}.

\bibitem[Liu et~al.(2024{\natexlab{a}})Liu, Liu, Wang,
  et~al.]{liu2024robobamba}
J.~Liu, M.~Liu, Z.~Wang, et~al.
\newblock {RoboMamba}: Efficient vision-language-action model for robotic
  reasoning and manipulation.
\newblock In \emph{Advances in Neural Information Processing Systems
  (NeurIPS)}, 2024{\natexlab{a}}.

\bibitem[Liu et~al.(2023)Liu, Desai, Liao, Wang, Xie, Xu, Kyrillidis, and
  Shrivastava]{liu2023scissorhands}
Z.~Liu, A.~Desai, F.~Liao, W.~Wang, V.~Xie, Z.~Xu, A.~Kyrillidis, and
  A.~Shrivastava.
\newblock Scissorhands: Exploiting the persistence of importance hypothesis for
  {LLM} {KV} cache compression at test time.
\newblock In \emph{Advances in Neural Information Processing Systems
  (NeurIPS)}, 2023.

\bibitem[Liu et~al.(2024{\natexlab{b}})Liu, Yuan, Jin, Zhong, Xu, Braverman,
  Chen, and Hu]{liu2024kivi}
Z.~Liu, J.~Yuan, H.~Jin, S.~Zhong, Z.~Xu, V.~Braverman, B.~Chen, and X.~Hu.
\newblock {KIVI}: A tuning-free asymmetric 2bit quantization for {KV} cache.
\newblock In \emph{International Conference on Machine Learning (ICML)},
  2024{\natexlab{b}}.

\bibitem[Mitra et~al.(2026)Mitra, Yuceel, Giles, and
  Pai]{factoreddiffusion2026}
S.~Mitra, E.~Yuceel, N.~Giles, and A.~Pai.
\newblock Factored diffusion policies: Compositionally generalized robot
  control with a single score network.
\newblock \emph{arXiv preprint arXiv:2605.22596}, 2026.
\newblock URL \url{https://arxiv.org/abs/2605.22596}.
\newblock Trajectory-tube closed-loop certificate for diffusion policy
  composition; certifies composition NOT memory sufficiency.

\bibitem[Morad et~al.(2023)Morad, Kortvelesy, Bettini, Liwicki, and
  Prorok]{morad2023popgym}
S.~Morad, R.~Kortvelesy, M.~Bettini, S.~Liwicki, and A.~Prorok.
\newblock {POPGym}: Benchmarking partially observable reinforcement learning.
\newblock In \emph{International Conference on Learning Representations
  (ICLR)}, 2023.

\bibitem[Moschella et~al.(2026)Moschella, Manduchi, and
  Sener]{moschella2026kvpolicy}
L.~Moschella, L.~Manduchi, and O.~Sener.
\newblock Learning to evict from key-value cache.
\newblock \emph{arXiv preprint arXiv:2602.10238}, 2026.

\bibitem[Moyer(2026)]{semieng2026hbf}
B.~Moyer.
\newblock Flash getting stacked high-bandwidth version.
\newblock Semiconductor Engineering, May 2026.
\newblock URL
  \url{https://semiengineering.com/flash-getting-stacked-high-bandwidth-version/}.
\newblock May 14, 2026.

\bibitem[Pathak et~al.(2017)Pathak, Agrawal, Efros, and Darrell]{pathak2017icm}
D.~Pathak, P.~Agrawal, A.~A. Efros, and T.~Darrell.
\newblock Curiosity-driven exploration by self-supervised prediction.
\newblock In \emph{International Conference on Machine Learning (ICML)}, 2017.

\bibitem[Peng et~al.(2023)Peng, Alcaide, Anthony, et~al.]{peng2023rwkv}
B.~Peng, E.~Alcaide, Q.~Anthony, et~al.
\newblock {RWKV}: Reinventing {RNN}s for the transformer era.
\newblock In \emph{Findings of the Association for Computational Linguistics:
  EMNLP 2023}, 2023.
\newblock \doi{10.18653/v1/2023.findings-emnlp.936}.
\newblock URL \url{https://aclanthology.org/2023.findings-emnlp.936}.

\bibitem[Qiu et~al.(2026)Qiu, Huang, and Ying]{qiu2026dysta}
W.~Qiu, T.~Huang, and R.~Ying.
\newblock Efficient long-horizon vision-language-action models via
  static-dynamic disentanglement.
\newblock \emph{arXiv preprint arXiv:2602.03983}, 2026.

\bibitem[Ramsauer et~al.(2021)Ramsauer, Sch{\"a}fl, Lehner, Seidl, Widrich,
  Adler, Gruber, Holzleitner, Pavlovi\'c, Sandve, Greiff, Kreil, Kopp,
  Klambauer, Brandstetter, and Hochreiter]{ramsauer2021hopfield}
H.~Ramsauer, B.~Sch{\"a}fl, J.~Lehner, P.~Seidl, M.~Widrich, T.~Adler,
  L.~Gruber, M.~Holzleitner, M.~Pavlovi\'c, G.~K. Sandve, V.~Greiff, D.~Kreil,
  M.~Kopp, G.~Klambauer, J.~Brandstetter, and S.~Hochreiter.
\newblock Hopfield networks is all you need.
\newblock In \emph{International Conference on Learning Representations
  (ICLR)}, 2021.

\bibitem[Roy et~al.(2005)Roy, Gordon, and Thrun]{roy2005beliefcomp}
N.~Roy, G.~Gordon, and S.~Thrun.
\newblock Finding approximate {POMDP} solutions through belief compression.
\newblock \emph{Journal of Artificial Intelligence Research}, 23:\penalty0
  1--40, 2005.
\newblock URL
  \url{http://www.cs.cmu.edu/~ggordon/roy-gordon-thrun.belief-compression-jair.pdf}.

\bibitem[Schlag et~al.(2021)Schlag, Irie, and
  Schmidhuber]{schlag2021fastweight}
I.~Schlag, K.~Irie, and J.~Schmidhuber.
\newblock Linear transformers are secretly fast weight programmers.
\newblock In \emph{International Conference on Machine Learning (ICML)}, 2021.

\bibitem[Smith et~al.(2023)Smith, Warrington, and Linderman]{smith2023s5}
J.~T. Smith, A.~Warrington, and S.~W. Linderman.
\newblock Simplified state space layers for sequence modeling.
\newblock \emph{arXiv preprint}, 2023.
\newblock URL \url{https://arxiv.org/abs/2208.04933}.

\bibitem[Sridhar et~al.(2025)Sridhar, Pan, Sharma, and Finn]{sridhar2025memer}
A.~Sridhar, J.~Pan, S.~Sharma, and C.~Finn.
\newblock {MemER}: Scaling up memory for robot control via experience
  retrieval.
\newblock \emph{arXiv preprint arXiv:2510.20328}, 2025.

\bibitem[Subramanian et~al.(2022)Subramanian, Sinha, Seraj, and
  Mahajan]{subramanian2022ais}
J.~Subramanian, A.~Sinha, R.~Seraj, and A.~Mahajan.
\newblock Approximate information state for approximate planning and
  reinforcement learning in partially observed systems.
\newblock \emph{Journal of Machine Learning Research}, 23\penalty0
  (12):\penalty0 1--83, 2022.
\newblock URL \url{https://jmlr.org/papers/v23/20-1165.html}.

\bibitem[Sun et~al.(2023)Sun, Dong, Huang, Ma, Xia, Xue, Wei,
  et~al.]{sun2023retnet}
Y.~Sun, L.~Dong, S.~Huang, S.~Ma, Y.~Xia, J.~Xue, F.~Wei, et~al.
\newblock Retentive network: A successor to {Transformer} for large language
  models.
\newblock \emph{arXiv preprint arXiv:2307.08621}, 2023.

\bibitem[Sun et~al.(2024)Sun, Li, Dalal, Xu, Vikram, Zhang, Dubois, Chen, Wang,
  Koyejo, Hashimoto, and Guestrin]{sun2024ttt}
Y.~Sun, X.~Li, K.~Dalal, J.~Xu, A.~Vikram, G.~Zhang, Y.~Dubois, X.~Chen,
  X.~Wang, S.~Koyejo, T.~Hashimoto, and C.~Guestrin.
\newblock Learning to (learn at test time): {RNNs} with expressive hidden
  states.
\newblock \emph{arXiv preprint arXiv:2407.04620}, 2024.

\bibitem[Swain et~al.(2026{\natexlab{a}})Swain, Han, Weidele, Martino, and
  Torralba]{tensorcache2026}
K.~Swain, S.~Han, D.~K.~I. Weidele, M.~Martino, and A.~Torralba.
\newblock Tensor cache: Eviction-conditioned associative memory for
  transformers.
\newblock \emph{arXiv preprint arXiv:2605.22884}, 2026{\natexlab{a}}.
\newblock URL \url{https://arxiv.org/abs/2605.22884}.
\newblock MIT/Torralba group; bounded fast-weight prior; LM-only,
  eviction-triggered write, no certificate.

\bibitem[Swain et~al.(2026{\natexlab{b}})Swain, Han, Weidele, Martino, and
  Torralba]{tensormemory2026}
K.~Swain, S.~Han, D.~K.~I. Weidele, M.~Martino, and A.~Torralba.
\newblock Tensor memory: Fixed-size recurrent state for long-horizon
  transformers.
\newblock \emph{arXiv preprint arXiv:2605.27686}, 2026{\natexlab{b}}.
\newblock URL \url{https://arxiv.org/abs/2605.27686}.
\newblock Fixed-size 3D recurrent tensor; spatial soft-write; perception loss;
  no control-rate, no certificate.

\bibitem[{TechTimes}(2026)]{techtimes2026dram}
{TechTimes}.
\newblock {DRAM} prices reach all-time high at \$20: {Q2} increase slows as
  {PC} deals close, May 2026.
\newblock URL
  \url{http://www.techtimes.com/articles/317403/20260530/dram-prices-reach-all-time-high-20-q2-increase-slows-pc-deals-close.htm}.
\newblock May 30, 2026; TrendForce / DRAMeXchange data.

\bibitem[Tishby et~al.(1999)Tishby, Pereira, and Bialek]{tishby2000ib}
N.~Tishby, F.~C. Pereira, and W.~Bialek.
\newblock The information bottleneck method.
\newblock In \emph{Proceedings of the 37th Annual Allerton Conference on
  Communication, Control and Computing}, pages 368--377, 1999.

\bibitem[Torne et~al.(2026)Torne, Pertsch, Walke, Vedder, Nair, Ichter, Ren,
  Wang, Tang, Stachowicz, Dhabalia, Equi, Vuong, Springenberg, Levine, Finn,
  and Driess]{mem2026}
M.~Torne, K.~Pertsch, H.~Walke, K.~Vedder, S.~Nair, B.~Ichter, A.~Z. Ren,
  H.~Wang, J.~Tang, K.~Stachowicz, K.~Dhabalia, M.~Equi, Q.~Vuong, J.~T.
  Springenberg, S.~Levine, C.~Finn, and D.~Driess.
\newblock {MEM}: Multi-scale embodied memory for vision language action models.
\newblock \emph{arXiv preprint arXiv:2603.03596}, 2026.
\newblock URL \url{https://arxiv.org/abs/2603.03596}.
\newblock Mixed-modal embodied memory (video+text); not O(1) VRAM; no
  action-gate; no certificate.

\bibitem[Tu et~al.(2024)Tu, Vashchilenko, Lu, and Xu]{tu2024vlcache}
D.~Tu, D.~Vashchilenko, Y.~Lu, and P.~Xu.
\newblock {VL-Cache}: Sparsity and modality-aware {KV} cache compression for
  vision-language model inference acceleration.
\newblock \emph{arXiv preprint arXiv:2410.23317}, 2024.

\bibitem[Xiao et~al.(2024)Xiao, Tian, Chen, Han, and
  Lewis]{xiao2024streamingllm}
G.~Xiao, Y.~Tian, B.~Chen, S.~Han, and M.~Lewis.
\newblock Efficient streaming language models with attention sinks.
\newblock In \emph{International Conference on Learning Representations
  (ICLR)}, 2024.

\bibitem[Xu et~al.(2025{\natexlab{a}})Xu, Wang, Xia, Zhu, Huang, and
  Xu]{xu2025vlacache}
S.~Xu, Y.~Wang, C.~Xia, D.~Zhu, T.~Huang, and C.~Xu.
\newblock {VLA-Cache}: Efficient vision-language-action manipulation via
  adaptive token caching.
\newblock In \emph{Advances in Neural Information Processing Systems
  (NeurIPS)}, 2025{\natexlab{a}}.

\bibitem[Xu et~al.(2025{\natexlab{b}})Xu, Zhuang, and Shan]{xu2025kvvla}
W.~Xu, L.~Zhuang, and L.~Shan.
\newblock {KV-Efficient VLA}: A method of speed up vision language model with
  {RNN}-gated chunked {KV} cache.
\newblock \emph{arXiv preprint arXiv:2509.21354}, 2025{\natexlab{b}}.

\bibitem[Yang et~al.(2024)Yang, Wang, Shen, Panda, and Kim]{yang2024gla}
S.~Yang, B.~Wang, Y.~Shen, R.~Panda, and Y.~Kim.
\newblock Gated linear attention transformers with hardware-efficient training.
\newblock In \emph{International Conference on Machine Learning (ICML)}, 2024.

\bibitem[Yang et~al.(2025)Yang, Wang, Wen, Luo, Zou, Zhang, Wen, and
  Zhang]{yang2025efficientvla}
Y.~Yang, Y.~Wang, Z.~Wen, Z.~Luo, C.~Zou, Z.~Zhang, C.~Wen, and L.~Zhang.
\newblock {EfficientVLA}: Training-free acceleration and compression for
  vision-language-action models.
\newblock \emph{arXiv preprint arXiv:2506.10100}, 2025.

\bibitem[{Zacks Investment Research}(2026)]{zacks2026micron}
{Zacks Investment Research}.
\newblock Micron stock slips despite blowout earnings, upbeat guidance, May
  2026.
\newblock URL
  \url{https://www.zacks.com/commentary/2886800/micron-stock-slips-despite-blowout-earnings-upbeat-guidance}.
\newblock May 22, 2026.

\bibitem[Zhang et~al.(2021)Zhang, McAllister, Calandra, Gal, and
  Levine]{zhang2021dbc}
A.~Zhang, R.~McAllister, R.~Calandra, Y.~Gal, and S.~Levine.
\newblock Learning invariant representations for reinforcement learning without
  reconstruction.
\newblock In \emph{International Conference on Learning Representations
  (ICLR)}, 2021.

\bibitem[Zhang et~al.(2025)]{zhang2025lact}
T.~Zhang et~al.
\newblock Test-time training done right.
\newblock \emph{arXiv preprint arXiv:2505.23884}, 2025.
\newblock Also available at OpenReview Tb9qAxT3xv.

\bibitem[Zhang et~al.(2023)Zhang, Sheng, Zhou, Chen, Zheng, Cai, Song, Tian,
  R{\'e}, Barrett, Wang, and Chen]{zhang2023h2o}
Z.~Zhang, Y.~Sheng, T.~Zhou, T.~Chen, L.~Zheng, R.~Cai, Z.~Song, Y.~Tian,
  C.~R{\'e}, C.~Barrett, Z.~Wang, and B.~Chen.
\newblock {H2O}: Heavy-hitter oracle for efficient generative inference of
  large language models.
\newblock In \emph{Advances in Neural Information Processing Systems
  (NeurIPS)}, 2023.

\end{thebibliography}

\appendix

\section{Proofs (full detail)}
\label{app:proofs}

\subsection{Proof of Theorem~\ref{thm:main} (full detail)}
\label{app:proof:main}

We reproduce the full four-step proof for self-containment.

\paragraph{Notation.}
Write the \textbf{true} discounted Bellman optimality operator on histories $\mathcal{B}$ and the
\textbf{surrogate} operator on $\calZ$, $\hat{\mathcal{B}}$:
\[
  (\mathcal{B}V)(h,a) = \EE[R\mid h,a] + \gamma\,\EE[V(H')\mid h,a],
  \qquad
  (\hat{\mathcal{B}}\hat V)(z,a) = \hat r(z,a) + \gamma\!\int_{\calZ}\!\hat V(z')\,\hat P(dz'\mid z,a).
\]
Both are $\gamma$-contractions in $\|\cdot\|_\infty$ (A1, A4); $V^*=\mathcal{B}V^*$,
$\hat V^*=\hat{\mathcal{B}}\hat V^*$. For a history $h_t$ we abbreviate $z_t = \sigma_t(h_t)$.

\paragraph{Step 1: dual (IPM) inequality.}
For any bounded $f$ and $\mu,\nu\in\Delta(\calZ)$, the Minkowski functional $\rho_F(f)$ (the
smallest $\rho>0$ with $\rho^{-1}f \in F$) satisfies
$\bigl|\int f\,d\mu - \int f\,d\nu\bigr| \le \rho_F(f)\, d_F(\mu,\nu)$. Apply with $f=\hat V^*$,
$\mu = \mu_t$ (true next-state law), $\nu = \nu_t = \hat P(\cdot\mid z_t,a_t)$. Using (AP2),
$d_F(\mu_t,\nu_t)\le\delta$:
\begin{equation}
  \Bigl|\,\EE[\hat V^*(Z_{t+1})\mid H_t=h_t,A_t=a_t] \;-\;
  \int \hat V^*(z')\,\hat P(dz'\mid z_t,a_t)\,\Bigr|
  \;\le\; \Lv\,\delta.
  \tag{A.1}
\end{equation}

\paragraph{Step 2: one-step Bellman mismatch.}
At any reachable $(h_t,a_t)$:
\begin{align}
  \bigl|(\mathcal{B}\hat V^*)(h_t,a_t) - (\hat{\mathcal{B}}\hat V^*)(z_t,a_t)\bigr|
  &\le \underbrace{\bigl|\EE[R_t\mid h_t,a_t]-\hat r(z_t,a_t)\bigr|}_{\le\,\varepsilon\text{ by (AP1)}}
  \;+\; \gamma\,\underbrace{\bigl|\EE[\hat V^*(Z_{t+1})\mid h_t,a_t]-\int \hat V^*\,d\nu_t\bigr|}_{\le\,\Lv\delta\text{ by (A.1)}} \notag \\
  &\le \varepsilon + \gamma\,\Lv\,\delta \;=:\; \eta.
  \tag{A.2}
\end{align}

\paragraph{Step 3: contraction propagates the one-step error.}
Let $\alpha = \sup_h |V^*(h) - \hat V^*(\sigma(h))|$ (finite by A1). Since $V^* = \mathcal{B}V^*$,
using (A.2) and the $\gamma$-contraction of $\mathcal{B}$:
\[
  \alpha \le \gamma\alpha + \eta, \quad\Rightarrow\quad
  \alpha \;\le\; \frac{\eta}{1-\gamma} \;=\; \frac{\varepsilon + \gamma\,\Lv\,\delta}{1-\gamma}.
  \tag{A.3}
\]
This proves part~(i).

\paragraph{Step 4: value approximation to closed-loop loss (factor of 2).}
Let $\pi_Z$ be greedy w.r.t.~$\hat Q^*$. For any state:
\[
  \Vstar(h) - \VpiZ(h)
  = \underbrace{\bigl(V^*(h) - \hat V^*(\sigma(h))\bigr)}_{\le\,\alpha\text{ by (A.3)}}
  + \underbrace{\bigl(\hat V^*(\sigma(h)) - \VpiZ(h)\bigr)}_{\le\,\alpha\text{ by policy-eval bound}}.
\]
The first term is $\le\alpha$ by part~(i). For the second, the deviation between the
surrogate-evaluated and truly-evaluated value of $\pi_Z$ is controlled by the one-step residual
$\eta$ propagated through the $\gamma$-contraction of the policy-evaluation Bellman operator,
giving $\le \eta/(1-\gamma)=\alpha$. Adding:
\[
  \bigl|\Vstar(h) - \VpiZ(h)\bigr| \;\le\; 2\alpha \;=\;
  \frac{2\,(\varepsilon + \gamma\,\Lv\,\delta)}{1-\gamma}.
  \tag{A.4}
\]
This factor of 2 matches \citet{subramanian2022ais} Thm~9/27 exactly. $\blacksquare$

\subsection{What $L_V$ is, per metric}
\label{app:lv}

The constant $L_V = \rho_F(\hat V^*)$ is the Minkowski functional of the surrogate value function
w.r.t.\ the IPM class $F$:
\begin{itemize}
  \item \textbf{Total variation} $(d_F{=}\TV)$: $L_V = \tfrac12\,\mathrm{span}(\hat V^*)
    = \tfrac12(\max\hat V^* - \min\hat V^*) \le \tfrac12\cdot\frac{R_{\max}-R_{\min}}{1-\gamma}$.
  \item \textbf{Wasserstein-1} $(d_F{=}W_1)$: $L_V = \|\hat V^*\|_{\mathrm{Lip}}$, the Lipschitz
    constant of the value function on $(\calZ,d)$.
\end{itemize}
In both cases $L_V<\infty$ under (A1)+(A3) and the bound reads $2(\varepsilon+\gamma L_V\delta)/(1-\gamma)$.
At current experimental scale, the numerically instantiated $L_V$-loaded bound is \textbf{vacuous}
(guaranteed form $52.69$ vs.\ trivial value span $10.0$; see \S\ref{sec:results:certificate}); we
report it transparently alongside the tighter $\Delta^*$ form (Remark~\ref{rem:tight}).

\section{Additional experimental details}
\label{app:experiments}

\subsection{Hyperparameter table}
\label{app:hyperparams}

\begin{table}[h]
\centering
\caption{Training hyperparameters (\texttt{TrainConfig}).}
\label{tab:hyperparams}
\small
\begin{tabular}{@{}lll@{}}
\toprule
\textbf{Parameter} & \textbf{Value} & \textbf{Description} \\
\midrule
\texttt{lr}               & $3\times10^{-3}$ & AdamW learning rate \\
\texttt{batch\_size}      & 64        & Batches per gradient step \\
\texttt{max\_steps}       & 4\,000    & Hard task sweep \\
\texttt{seq\_len}         & 96 / 128  & Sequence length (main / hard task) \\
\texttt{beta}             & $10^{-3}$ & IB loss weight \\
\texttt{gamma}            & $3\times10^{-3}$ & Write-sparsity loss weight \\
\texttt{gamma\_warmup\_frac} & 0.6   & Ramp $\gamma$ over first 60\% of steps \\
\texttt{write\_target\_rho} & 0.15  & Target mean write rate \\
\texttt{d\_model}         & 64        & Token embedding dimension \\
\texttt{d\_key}           & 16/32/64  & Fast-weight key dim (swept) \\
\texttt{d\_val}           & $=$ d\_key  & Fast-weight value dim \\
\texttt{enc\_hidden}      & 64        & Encoder MLP hidden size \\
\texttt{gate\_hidden}     & 64        & Gate MLP hidden size \\
\texttt{control\_hz}      & 20.0      & Control frequency (fixed constant) \\
Gradient clip max\_norm   & 1.0       & \texttt{clip\_grad\_norm\_} \\
EMA teacher decay         & 0.95      & teacher update rate \\
Optimizer betas           & (0.9, 0.999) & AdamW defaults \\
Seeds                     & $\{0,\ldots,6\}$ & up to 7 seeds per cell \\
\bottomrule
\end{tabular}
\end{table}

\subsection{Parameter-count disclosure}
\label{app:nparams}

Total parameter counts by variant and state budget (verified by \texttt{report\_params.py}; run
\texttt{lean-20260530-1449}). Parameters scale with the state dimension $N{=}d_k{=}d_v$; the five
primary parity variants (\methodname, \texttt{fixed\_size\_state}, \texttt{write\_every\_step},
\texttt{random\_write}, \texttt{periodic\_write}) match \emph{exactly} at every budget (identical
totals, not a $\pm$tolerance band). \texttt{full\_recurrence}, \texttt{learned\_token\_gate}, and
\texttt{no\_memory} differ by architecture and are disclosed separately.

\begin{table}[h]
\centering
\caption{Total parameter counts by variant across state budgets $N{=}d_k{=}d_v\in\{16,32,64\}$.}
\label{tab:params}
\small
\begin{tabular}{@{}lccc@{}}
\toprule
\textbf{Variant} & \textbf{$N{=}16$} & \textbf{$N{=}32$} & \textbf{$N{=}64$} \\
\midrule
\texttt{ours} (\methodname)         & 20{,}935 & 30{,}279 & 55{,}111 \\
\texttt{fixed\_size\_state}         & 20{,}935 & 30{,}279 & 55{,}111 \\
\texttt{write\_every\_step}         & 20{,}935 & 30{,}279 & 55{,}111 \\
\texttt{random\_write}              & 20{,}935 & 30{,}279 & 55{,}111 \\
\texttt{periodic\_write}            & 20{,}935 & 30{,}279 & 55{,}111 \\
\texttt{full\_recurrence}           & 26{,}061 & 30{,}125 & 44{,}397 \\
\texttt{learned\_token\_gate}       & 22{,}100 & 31{,}444 & 56{,}276 \\
\texttt{no\_memory}                 & 7{,}316  & 11{,}556 & 26{,}180 \\
\bottomrule
\end{tabular}
\end{table}

The gradient-active counts at the $N{=}32$ reference are 21{,}447 for \methodname\ versus 15{,}110
for the scheduled-write parity variants (\texttt{write\_every\_step}, \texttt{fixed\_size\_state},
\texttt{random\_write}, \texttt{periodic\_write}), the $+6{,}337$ \texttt{gate\_mlp} asymmetry
discussed below.

\textbf{Predictor-head disclosure.} The \texttt{GatedTTTState.predictor} sub-head (8{,}320
parameters at this size) is \emph{allocated but never called} in any variant's forward pass; a
backward-pass audit confirms its gradient is \texttt{None} for every variant. It inflates all
nominal counts by 8{,}320 without contributing computation; it should be removed before final
checkpoint release.

\textbf{Gradient-active asymmetry.} \methodname\ has 21{,}447 gradient-active parameters versus
15{,}110 for the scheduled-write variants, a difference of $+6{,}337$ ($+41.9\%$), because the
\texttt{gate\_mlp} is gradient-active only in \methodname. The framing ``ours $=$ write-every-step
minus the gate'' is architecturally accurate but understates this capacity advantage
(\S\ref{sec:limitations}, item 4).

\textbf{GRU configuration.} \texttt{full\_recurrence} uses \texttt{gru\_hidden}$=53$ ($-0.52\%$ vs.\
ours nominal) to correct a prior configuration where \texttt{hidden}$=32$ left the GRU materially
under-parameterized; any results using \texttt{hidden}$=32$ were invalid and discarded.

\subsection{Reproducibility commands}
\label{app:repro}

\begin{verbatim}
# Deploy sweep app (one-time):
modal deploy aura/lean_sweep.py

# Launch main hard-task sweep:
modal run --detach aura/lean_sweep.py --max-concurrent 10

# Collect and aggregate results:
python -m aura.lean_aggregate --run-tag lean-YYYYMMDD-HHMM

# Check proof gates:
python aura/check_proof_matrix.py --run-tag lean-YYYYMMDD-HHMM

# Single-seed reference run:
python -m aura.train --variant ours --task noisy_long_recall \
  --seeds 0 1 2 --max-steps 4000 --seq-len 96 \
  --batch-size 64 --lr 3e-3 --beta 1e-3 --gamma 3e-3 \
  --write-target-rho 0.15 --state-dim 32 --device cuda
\end{verbatim}

Source code and checkpoint: \url{https://huggingface.co/Kaikaku/aura} (currently a private repository; access on request).

\subsection{Extended horizon-stress detail}
\label{app:horizonstress}

The main-text crossover figure (Figure~\ref{fig:mem_growth}) shows the constant-vs-linear
carried-state separation near the practically relevant short-horizon regime. For completeness we
include the full 100{,}000-step horizon-stress measurement below
(Figure~\ref{fig:vram}), which extends the same constant-vs-growing comparison out to the long
horizon and quantifies the $6{,}061\times$ separation; it complements, rather than duplicates, the
main-text figure.

\begin{figure}[h]
  \centering
  \includegraphics[width=0.72\linewidth]{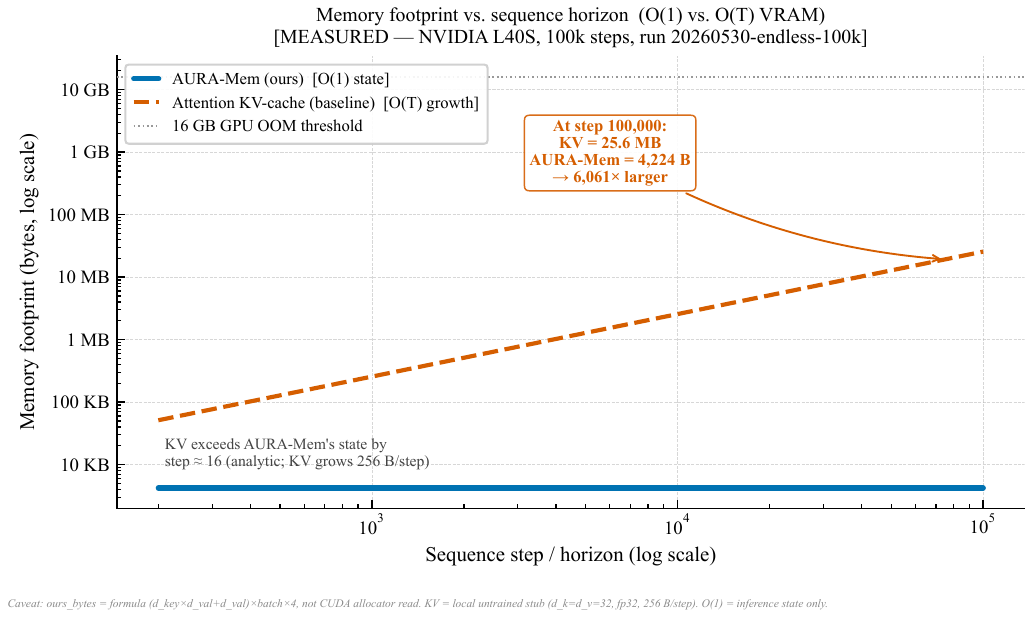}
  \caption{%
    \textbf{Extended 100k-step horizon-stress detail: memory footprint vs.\ sequence horizon}
    (log--log; NVIDIA L40S, run \texttt{20260530-endless-100k}, 100{,}000 steps, 500 logged
    checkpoints). This is the extended long-horizon companion to the main-text crossover figure
    (Figure~\ref{fig:mem_growth}): the same constant-vs-growing comparison carried out to
    100{,}000 steps. \methodname's inference state (blue) occupies a \emph{constant}
    $\mathbf{4{,}224}$ \textbf{bytes} at every step, computed by the closed-form formula
    $(d_k d_v + d_v){\times}\text{batch}{\times}4$ ($d_k{=}d_v{=}32$, batch$=1$, fp32) and confirmed
    identical across all 500 logged steps, while a growing-KV reference stub (vermillion dashed;
    same dimensions, local untrained reference, $256$ bytes/step analytic) reaches
    $\mathbf{25{,}600{,}000}$ \textbf{bytes} (25.6\,MB) at step 100{,}000, a ratio of
    $\mathbf{6{,}061\times}$ larger. (O(1) refers to the inference state only; see
    \S\ref{sec:method:state}.) The 4{,}224-byte figure is the architectural state formula and
    must not be conflated with the peak CUDA allocator reading
    (\texttt{cuda\_max\_memory\_allocated} $= 43{,}051{,}008$ bytes, which includes weights,
    activations, and transient buffers); the KV contrast uses a local untrained stub, not a
    trained transformer.
  }
  \label{fig:vram}
\end{figure}

\section{Model card (abridged)}
\label{app:modelcard}

\paragraph{Intended use.}
Research benchmarking of bounded-memory policies on synthetic tasks; ablations on write-gating, IB
compression, and bandwidth efficiency; long-horizon control on memory-constrained hardware
(proof-of-concept, research only). \textbf{Not validated for} real-robot deployment, safety-critical
or unmonitored use, or non-embodied workloads.

\paragraph{Configuration.}
The released checkpoint (\texttt{Kaikaku/aura}) uses $d_k{=}d_v{=}64$ (55{,}111
parameters, 16{,}640-byte inference state) and differs from the sweep configuration ($d_k{=}d_v{=}N$,
30{,}279 parameters at the $N{=}32$ reference) used for all experimental results; the two must not be cross-cited
(\S\ref{sec:limitations}, item 7).

\paragraph{Limitations (required disclosure).}
All evaluation is on synthetic benchmarks (\texttt{noisy\_long\_recall}, \texttt{sparse\_recall});
no real-robot validation. The $\varepsilon$ and $\delta$ values are action self-consistency
diagnostics, \emph{not} formal theorem-certificate values, and the instantiated value-loss bound is
vacuous at current scale.

\paragraph{Citation.}
\begin{verbatim}
@article{josefchen2026auramem,
  title   = {Write Only What You'd Act On: Learned Action-Error Gating
             for O(1)-VRAM Robot Policies},
  author  = {Josef Chen},
  journal = {Preprint},
  year    = {2026},
  url     = {https://huggingface.co/Kaikaku/aura}
}
\end{verbatim}

\section{Broader impact}
\label{app:impact}

Reducing inference-state VRAM and memory-write traffic for robot policies lowers the barrier to
on-device deployment of capable vision-language-action models on memory-constrained edge hardware.
Because every autoregressive step issues a memory write, and because high-bandwidth memory is the
scarce, expensive resource currently bottlenecking physical-AI deployment, a $4.98$--$9.19\times$
reduction in writes per second suggests a proportional reduction in DRAM/HBM write traffic, a
first-order driver of energy cost on LPDDR/HBM hardware. We emphasize, however, that we measure
\emph{write counts}, not joules: any energy implication is a proxy argument, and measured
energy savings would require hardware-level experiments beyond this paper's scope. We make no
measured-energy claims.

The principal risk is the usual dual-edged consequence of efficiency: gains that make capable
policies cheaper to deploy can also accelerate deployment into unmonitored or safety-critical
settings before adequate validation. \methodname\ is validated only on synthetic benchmarks and is
explicitly not certified for real-robot or safety-critical use; deployers should treat it as a
research artifact requiring appropriate human oversight and task-specific safety validation. We
identify no application-specific dual-use concern beyond this general efficiency consideration: the
contribution is a general memory-efficiency mechanism rather than a capability targeted at a
sensitive domain.

\end{document}